\documentclass[10pt]{article}
\usepackage[margin=1in]{geometry}
\usepackage{iftex}
\ifPDFTeX\else
  \usepackage{fontspec}
\fi
\usepackage{microtype}
\usepackage{amsmath,amssymb,mathtools,amsthm}
\usepackage{booktabs}
\usepackage{xcolor}
\usepackage[numbers,sort&compress]{natbib}
\usepackage[colorlinks=true,allcolors=blue!55!black]{hyperref}
\usepackage[ruled,vlined,linesnumbered]{algorithm2e}
\newtheorem{theorem}{Theorem}
\newtheorem{lemma}[theorem]{Lemma}
\newtheorem{proposition}[theorem]{Proposition}

\numberwithin{equation}{section}
\numberwithin{theorem}{section}
\newcommand{\Reg}{\mathsf{Reg}}
\newcommand{\inner}[2]{\left\langle #1,#2\right\rangle}
\newcommand{\E}{\mathbb E}
\newcommand{\R}{\mathbb R}
\newcommand{\ind}{\mathbf 1}
\newcommand{\ellhat}{\widehat{\ell}}

\newcommand{\diag}{\operatorname{diag}}
\newcommand{\poly}{\operatorname{poly}}
\newcommand{\argmin}{\operatorname*{arg\,min}}

\newcommand{\otil}{\widetilde{\mathcal{O}}}
\newcommand{\order}{\mathcal{O}}
\newcommand{\firstorder}{Q_{\infty,1}}
\newcommand{\secondorder}{Q_{\infty,2}}
\DontPrintSemicolon
\RestyleAlgo{ruled}
\SetAlgoVlined
\SetKwInput{KwInput}{Input}
\SetKwInput{KwOutput}{Output}

\title{Toward Optimal Second-Order Path-Length Guarantee for Adversarial Multi-Armed Bandits}
\author{  Mengxiao Zhang \\
  University of Iowa \\  \texttt{mengxiao-zhang@uiowa.edu} \\}
\date{}

\begin{document}
\maketitle

\begin{abstract}

We study second-order path-length regret in adversarial
$K$-armed bandits against oblivious loss sequences.  \citet{bubeck2019path} designed an algorithm that achieves $\otil(K+\sqrt{K\firstorder})$ regret, where $\firstorder$ is the first-order path length, and left open whether $\otil(\poly(K)\sqrt{1+\secondorder})$ regret is achievable under bandit feedback, where $\secondorder$ is the second-order path length. Somewhat surprisingly, we resolve this question positively by showing that with a more involved analysis, the exact same algorithm of \citet{bubeck2019path} achieves $\order\left(K\log(KT)+\sqrt{K\log(KT)\bigl(1+\secondorder\bigr)}\right)$ expected regret when $\secondorder$ is known, where $T$ is the horizon. This matches the $\Omega(\sqrt{KQ_{\infty,2}})$ lower bound up to logarithmic factors and additive terms.  We further remove the knowledge of $\secondorder$ using an adaptive restart scheme whose path-length estimator has uniformly bounded increments.
\end{abstract}

\section{Introduction}\label{sec:intro}

The adversarial multi-armed bandit problem is a canonical model of
sequential decision making under partial feedback. At the beginning of the game, the environment secretly decides a loss sequence $\ell_1,\ell_2,\dots,\ell_T\in[0,1]^K$ for $T$ rounds where $K$ is the number of actions.\footnote{This environment is also commonly referred to as an oblivious adversary.} In each of $T$ rounds,
a learner chooses one of $K$ actions, or {arms}. Then, she incurs
and observes only the selected coordinate $\ell_{t,I_t}$; the other
coordinates remain hidden.  Performance is measured against the best single
arm in hindsight.  This elementary protocol captures the central tension
between exploration and exploitation and has consequently served as a basic
model for online recommendation, routing, allocation, and repeated decision
making.  Without further structure on the loss sequence, the minimax regret
is $\Theta(\sqrt{KT})$ \citep{auer2002nonstochastic,zimmert2021tsallis}.

The minimax guarantee is derived for an environment that may
change arbitrarily at every round, and can therefore be pessimistic on more
regular data.  A substantial literature develops adaptive guarantees that
replace the horizon by an observable or instance-dependent measure of
difficulty.  Examples include first-order bounds governed by the loss of a
good comparator \citep{wei2018adaptive,pogodin2020first},
best-of-both-worlds guarantees that are minimax in adversarial environments
and logarithmic in stochastic ones \citep{bubeck2012best,zimmert2021tsallis},
and bounds depending on empirical variance or related second-order
quantities \citep{hazan2011benign,bubeck2018sparsity,ito2022robust}.  These
results express the same broad principle: the regret should depend on the
complexity actually present in the loss sequence rather than automatically
scaling as the worst-case rate $\sqrt{KT}$.

This paper concerns adaptivity to temporal variation, also called the path
length of the loss sequence.  For loss vectors
$\ell_1,\ldots,\ell_T\in[0,1]^K$, define the first- and second-order path
lengths by
$$
 \firstorder\triangleq\sum_{t=2}^T\|\ell_t-\ell_{t-1}\|_\infty,
 \qquad
 \secondorder\triangleq\sum_{t=2}^T\|\ell_t-\ell_{t-1}\|_\infty^2.
$$
The first quantity measures movement linearly, whereas the second measures it
quadratically.  This distinction is substantial when the environment moves
through many small increments: if each increment has infinity norm
$T^{-1/2}$, then $\secondorder=\order(1)$ although the corresponding
first-order movement can be $\order(\sqrt T)$.

For adversarial MAB, \citet{wei2018adaptive} proved the guarantee
$\otil(K+\sqrt{KQ_{1,1}})$, where $Q_{1,1}=\sum_{t=2}^T\|\ell_t-\ell_{t-1}\|_1$. \footnote{We use $\otil(\cdot)$ to suppress poly-logarithmic terms in $K$ and $T$.} \citet{bubeck2019path} subsequently
introduced a recent-arm-biased optimistic mirror-descent algorithm and proved
the sharper guarantee of $\otil(K+\sqrt{K\firstorder})$, where
$\firstorder=\sum_{t=2}^T\|\ell_t-\ell_{t-1}\|_\infty\le Q_{1,1}$.\footnote{\citet{bubeck2019path} in fact prove this guarantee against a stronger adaptive adversary. They also propose another algorithm achieving $\otil(K^{1/3}Q_{1,1}^{1/3}T^{1/6}+K^2)$ for an oblivious adversary.}
Both guarantees nevertheless depend linearly on the magnitude of each
temporal increment.  This is different from the full-information experts
setting, in which the complete vector $\ell_t$ is observed after every
round.  In that setting, \citet{steinhardt2014adaptivity} showed that
second-order path-length dependence is achievable.  More recently,
\citet{chen2021impossible} resolved the associated ``impossible tuning''
problem: a single expert algorithm can adapt simultaneously to
comparator-specific second-order prediction errors without knowing their
scales in advance.  Thus full information permits such second-order
path-length guarantees,
while the corresponding possibility under bandit feedback remained open.
This leads to the following question left by
\citet{wei2018adaptive,bubeck2019path}:

\begin{center}
\emph{Is $\otil\!\left(\poly(K)(1+\sqrt{\secondorder})\right)$ regret
achievable for adversarial multi-armed bandits?}
\end{center}

\paragraph{Contributions.}
We answer the question affirmatively for oblivious loss
sequences.  Somewhat surprisingly, when $\secondorder$ is known, we prove that \emph{the exact same
algorithm} proposed by \citet{bubeck2019path}, without any change to its
sampler, estimator, or mirror-descent update, achieves the regret of
$$
 \order\left(K\log(KT)+
 \sqrt{K\log(KT)(1+\secondorder)}\right),
$$
under oracle tuning. The contribution in Section~\ref{sec:known} is therefore a new analysis, not
a new algorithm.  Standard OMD reduces the proof to a stability divergence
plus the bias caused by recent-arm sampling.  Instead of taking an absolute
value and bounding it by a first-order movement term, as in \citet{bubeck2019path},
we combine the bias exactly with the prediction-error square and retain a
signed difference term. This refined analysis leads to the desired second-order path-length guarantee. We then remove knowledge of $\secondorder$ in
Section~\ref{sec:parameter-free} using a different, adaptively tuned
algorithm with a modified sampler and a bounded-increment path-length estimator. The resulting parameter-free regret is
$\order(K\log(KT)+\log T+\sqrt{K\secondorder\log(KT)\log T})$, which is a factor of $\sqrt{\log T}$ worse than the one with oracle tuning. 

\paragraph{Related work.}
Under full-information feedback, path-length guarantees were developed through gradual-variation and predictable-sequence analyses where
optimistic online learning can compare the current loss vector with a predictable hint and obtain regret whose data-dependent term is governed by the prediction error \citep{chiang2012gradual,rakhlin2013predictable}. Using the
previous loss vector as the hint turns the squared prediction error into squared temporal variation.  \citet{steinhardt2014adaptivity} developed an adaptive optimistic exponentiated-gradient algorithm and obtained variance and
squared-path regret guarantees for the best expert. \citet{chen2021impossible}
later solved the impossible-tuning problem for experts by combining
mirror descent, a correction term, and a weighted entropy regularizer.  Their
algorithm adapts simultaneously to the second-order prediction error of every
expert. Both results use the full loss vector. They therefore do not
directly tune the hidden squared path length in ordinary MAB, where only one coordinate is observed per round and the squared increments are themselves hidden.

The bandit path-length literature replaces the missing vector observation by
importance weighting.  \citet{wei2018adaptive} developed the BROAD-OMD
framework and obtained bounds involving the first-order movement of the best
arm or of all arms.  The latter comes with a negative stability term that can cancel the opponents' movement in repeated games. This gives faster convergence of the average play to equilibrium in two-player zero-sum games with bandit feedback, continuing the full-information connection between no-regret learning and game
dynamics \citep{syrgkanis2015fast}.  \citet{bubeck2019path} improved the MAB
dependence from the $\ell_1$ path measure to the smaller $\ell_\infty$ path measure by using a common optimistic prediction and biasing play toward the most recently selected arm.  Our oracle-tuned result uses precisely that
algorithm.  What changes is the analysis: their proof bounds the bias by a first-order selected-coordinate movement, whereas our signed identity and contraction reduce the same algorithm to $\secondorder$. 

Beyond temporal variation, adaptive adversarial-bandit guarantees have been
developed for small comparator loss
\citep{wei2018adaptive,pogodin2020first}, loss variance
\citep{hazan2011benign,bubeck2018sparsity,ito2022robust}, sparse loss vectors
\citep{kwon2016gains,bubeck2018sparsity}, and best-of-both-worlds adaptation
\citep{bubeck2012best,zimmert2021tsallis,jin2023improved,lee2024follow}. These guarantees are generally
incomparable because they exploit different structure, including cumulative magnitude,
dispersion, support size, or stochastic separation. None of these is
characterized by temporal smoothness alone. Related instance-adaptive
guarantees also extend beyond ordinary MABs, including bandits with feedback
graphs \citep{lykouris2018small,lee2020graph}, adversarial MDPs
\citep{lee2020bias,jin2020simultaneously,jin2021best}, linear bandits
\citep{bubeck2019path,ito2023linear}, combinatorial semi-bandits
\citep{neu2015first,wei2018adaptive}, and bandit convex optimization
\citep{vanderhoeven2020comparator,yu2026gradient}.

\section{Notations and Problem Setting}
\label{sec:preliminaries}

\paragraph{Notation.}
For $K\ge2$, write $[K]=\{1,\ldots,K\}$. Let $e_i$ be the $i$th standard
basis vector and for a vector $v\in\R^K$, denote its $i$th coordinate as $v_i$. Let $\Delta_K=\{p\in\R^K: \sum_{i=1}^Kp_i=1,~p_i\geq 0~\forall i\in[K] \}$ be the probability simplex. Inner products
and Euclidean norms are denoted by $\langle\cdot,\cdot\rangle$ and
$\|\cdot\|_2$.  For a differentiable convex function $\Psi$, its Bregman
divergence is $D_\Psi(x,y)\triangleq\Psi(x)-\Psi(y)-\langle\nabla\Psi(y),x-y\rangle$.
For $z\in\mathbb R^K$, the notation $z^{\odot2}$ means coordinatewise
squaring: $(z^{\odot2})_i=z_i^2$ for all $i\in[K]$. By convention, we use $0/0=0$.

\paragraph{Problem setting.}
Before interaction, an oblivious adversary fixes loss vectors
$\ell_1,\ldots,\ell_T\in[0,1]^K$.  At round $t$, the learner selects
$I_t\in[K]$ using only its internal randomness and the history
$\mathcal F_{t-1}=\sigma(I_1,c_1,\ldots,I_{t-1},c_{t-1})$, incurs loss
$c_t=\ell_{t,I_t}$, and observes only $c_t$.  The performance criterion is
expected static (pseudo) regret
\begin{equation}
 \Reg_T\triangleq\max_{i\in[K]}\Reg_T(i)\triangleq\max_{i\in[K]}\E\left[\sum_{t=1}^T\ell_{t,I_t}
 -\sum_{t=1}^T\ell_{t,i}\right].
 \label{eq:regret-definition}
\end{equation}
All expectations in the upper bounds are over the learner's internal randomization (and over the adversary's initial randomization if it draws the oblivious loss table from a distribution). The goal is to obtain an adaptive regret bound depending on the second-order path length $\secondorder\triangleq \sum_{t=2}^T\|\ell_t-\ell_{t-1}\|_\infty^2.$ For conciseness, we also write $\E_t[\cdot]=\E[\cdot\mid\mathcal F_{t-1}]$ when conditioning on all the randomness before $I_t$ is drawn.

\paragraph{Lower bound.}
Adapting the binary-loss construction used in \citet{bubeck2019path}, a lower bound of $\Omega(\sqrt{KQ_{\infty,2}})$ can be directly derived by the standard lower bound instance in adversarial MAB. We include this lower bound in the following for completeness.

\begin{proposition}
\label{prop:squared-path-lower-bound}
For every
$K\ge2$, $T\ge2$, and $2K\le q\le T$, every algorithm has to suffer $\Omega(\sqrt{Kq})$ regret under certain oblivious loss sequence $\ell_{1},\ell_2,\dots,\ell_T\in[0,1]^K$ where $Q_{\infty,2}\leq q$.
\end{proposition}

\begin{proof}
Let $n=\lfloor q\rfloor$.  By the standard adversarial MAB lower bound
\citep[Theorem~5.1]{auer2002nonstochastic}, for any algorithm, one can find
an oblivious binary loss sequence of length $n$ on which the algorithm
suffers $\Omega(\sqrt{Kn})$ expected regret.  Following
\citet[Section~2]{bubeck2019path}, use this sequence during the first $n$
rounds and set all losses to zero afterward. Since the losses are binary, every
$\|\ell_t-\ell_{t-1}\|_\infty^2$ is at most one.  There are at most $n$
transitions before the loss sequence becomes identically zero, and hence
$Q_{\infty,2}\le n\le q$.  Finally, $q\ge2K\ge2$ implies
$n=\lfloor q\rfloor\ge q/2$, so the regret is
$\Omega(\sqrt{Kn})=\Omega(\sqrt{Kq})$.
\end{proof}

\section{Known \texorpdfstring{$\secondorder$}{Q-infinity-2}: a Refined Analysis of
\texorpdfstring{\citet{bubeck2019path}}{Bubeck et al. (2019)}}
\label{sec:known}

In this section, we first consider the case in which the learner knows the
value of $\secondorder$. In Section~\ref{sec:parameter-free}, we will remove this
restriction using a more adaptive algorithm. Somewhat surprisingly, no new
algorithm is needed when $\secondorder$ is known, and we show that Algorithm~1 of
\citet{bubeck2019path} in fact already achieves the desired guarantee when its
learning rate is chosen using $\secondorder$.  Our contribution is then a more refined analysis of that algorithm.

For completeness, here we briefly introduce Algorithm~1 of \citet{bubeck2019path}. The algorithm maintains two distributions $p_t$ and $x_t$ at each round $t$. Specifically, the algorithm first selects an action $I_t$ from distribution $p_t$. After observing the loss $c_t\triangleq \ell_{t,I_t}$, the algorithm constructs an unbiased loss estimator $\ellhat_{t,i}$ using both $c_{t-1}$, the loss suffered in the previous round, and $c_t$. A direct calculation shows that $\E_t[\widehat\ell_{t,i}]=\ell_{t,i}$ for every $i\in[K]$ and more importantly, the variance of this estimator can be much smaller when the path length of the loss vectors is small. Next, the algorithm computes $x_{t+1}$ by applying a step of online mirror descent with log-barrier regularizer using $\ellhat_t$. Finally, the next-round strategy $p_{t+1}$ is obtained by biasing $x_{t+1}$ toward the arm $I_t$ just selected, with weight determined by $\lambda_{t+1}=\alpha(1-c_t)$. Algorithm~\ref{alg:known} shows the full pseudocode for completeness.

\begin{algorithm}[H]
\caption{Oracle-tuned Algorithm~1 of \citet{bubeck2019path}}
\label{alg:known}
\KwInput{$K,T$,$\eta>0$}
Set $x_1= p_1= \frac{1}{K}\mathbf1$, $\alpha=8\eta$, and $c_0=0$\;
Set $\Psi(x)\leftarrow\eta^{-1}\sum_{i=1}^K\log(1/x_i)$\;
\For{$t=1,\ldots,T$}{
  Draw $I_t\sim p_t$ and observe $c_t\leftarrow\ell_{t,I_t}$\;
  Set $\widehat\ell_{t,i}\leftarrow c_{t-1}+
  \ind\{I_t=i\}(c_t-c_{t-1})/p_{t,i}$ for every $i\in[K]$\;
  Set $x_{t+1}\leftarrow\argmin_{x\in\Delta_K}
  \left(\inner{x}{\widehat\ell_t}+D_\Psi(x,x_t)\right)$\;
  Set $\lambda_{t+1}\leftarrow\alpha(1-c_t)$ and
  $p_{t+1}\leftarrow(x_{t+1}+\lambda_{t+1}e_{I_t})/
  (1+\lambda_{t+1})$\;
}
\end{algorithm}

\citet{bubeck2019path} shows that Algorithm~\ref{alg:known} achieves $\otil(K+\sqrt{K\firstorder})$ regret where $\firstorder=\sum_{t=2}^T\|\ell_t-\ell_{t-1}\|_\infty$ is the first-order path-length that can be much larger than $\secondorder$. However, we show in the following theorem that when $\eta$ and $\alpha$ are chosen properly based on the knowledge of $\secondorder$, Algorithm~\ref{alg:known} in fact achieves the desired $\otil(K+\sqrt{K\secondorder})$ regret bound.

\begin{theorem}
\label{thm:known}
For every $K\ge2$, $T\ge2$, and oblivious loss sequence
$\ell_1,\ldots,\ell_T\in[0,1]^K$, Algorithm~\ref{alg:known} with
$\eta=\min\left\{\frac{1}{162},\sqrt{\frac{K\log(KT)}{1+\secondorder}}\right\}$ guarantees that
\begin{equation}
 \Reg_T\le \order\left(K\log(KT)+
 \sqrt{K\log(KT)(1+\secondorder)}\right).
 \label{eq:known-tuned-rate}
\end{equation}
\end{theorem}

To our knowledge, Theorem~\ref{thm:known} gives the first
$\otil(\poly(K)(1+\sqrt{\secondorder}))$ guarantee for adversarial MAB, resolving the open problems proposed in~\citet{wei2018adaptive,bubeck2019path}. In the following, we will show the proof of Theorem~\ref{thm:known}, which is our key novelty.

\subsection{Analysis of Algorithm~\ref{alg:known}}
\label{subsec:known-analysis}

We begin with the standard OMD analysis and highlight how our analysis is different from \citet{bubeck2019path}. Define $D_t\triangleq D_\Psi(x_t,x_{t+1})$ as the Bregman divergence between consecutive OMD updates $x_t$ and $x_{t+1}$, and
$B_t\triangleq\inner{p_t-x_t}{\ell_t}$ as the loss difference between $x_t$ and $p_t$ at round $t$. For each arm $i\in[K]$, set
$\gamma=1/(KT)$ and define the interior comparator
$u^{(i)}=(1-(K-1)\gamma)e_i+\gamma\sum_{j\ne i}e_j$. The OMD first-order condition and the
three-point identity (e.g. \citet[Lemma~6]{wei2018adaptive}) imply that
\begin{equation}
 \inner{x_t-u^{(i)}}{\widehat\ell_t}
 \le D_\Psi(u^{(i)},x_t)-D_\Psi(u^{(i)},x_{t+1})+D_t.
 \label{eq:omd-one-step-preview}
\end{equation}
Summing \eqref{eq:omd-one-step-preview}, taking expectations, using
$\E_t[\widehat\ell_t]=\ell_t$, and adding the sampling bias gives, for every
$i\in[K]$,
\begin{align}
 \Reg_T(i)
 &\le1+\E\left[\sum_{t=1}^T\inner{p_t-u^{(i)}}{\ell_t}\right]\notag\\
 &=1+\E\left[\sum_{t=1}^TB_t\right]
 +\E\left[\sum_{t=1}^T\inner{x_t-u^{(i)}}{\widehat\ell_t}\right]\notag\\
 &\le1+D_\Psi(u^{(i)},x_1)+
 \E\left[\sum_{t=1}^T(D_t+B_t)\right].
 \label{eq:known-omd-decomposition}
\end{align}
A direct calculation from $x_1=\frac{1}{K}\mathbf1$ gives
$D_\Psi(u^{(i)},x_1)\le \order(K\log(KT)/\eta)$, uniformly over $i\in[K]$.  Therefore, it remains to
prove that
\begin{equation}
 \E\left[\sum_{t=1}^T(D_t+B_t)\right]
 \le \order(1+\eta\secondorder).
 \label{eq:known-ledger-target}
\end{equation}

This is the exact point at which our analysis departs from
\citet[proof of Theorem~2]{bubeck2019path}. We briefly discuss how they handle this term. Specifically, their OMD calculation first upper-bounds $\E[\sum_{t=1}^TD_t]$ by
$4\eta\E[\sum_{t=1}^T(c_t-c_{t-1})^2]$. Then, their subsequent control of the sampling bias $\E[\sum_{t=1}^TB_t]$ contains the negative of the same term shown above and an additional first-order quantity
\begin{equation}
 \alpha\E\left[\sum_{t=2}^T
 \left|\ell_{t,I_{t-1}}-\ell_{t-1,I_{t-1}}\right|\right].
 \label{eq:bubeck-first-order-step}
\end{equation}
The two residual-square terms cancel, leaving
\eqref{eq:bubeck-first-order-step}. This is precisely why their argument produces a first-order path length. 

However, we do not make that relaxation in bounding $\E[\sum_{t=1}^T(D_t+B_t)]$.  Instead, we keep the exact divergence
$D_t$, whose coefficient is strictly smaller than the available negative
residual square as we will see later, and rewrite $B_t$ by an equality which replaces
the first-order term in \eqref{eq:bubeck-first-order-step} by a squared
movement plus a signed difference term.

The next four lemmas implement this argument in three steps: an exact
expansion of $B_t$, the stability control of the OMD movement $x_t$, and control of the
recent-arm correction. For the remainder of the section, we define
$r_t\triangleq c_t-c_{t-1}$,
$v_t\triangleq \ell_{t,I_{t-1}}-\ell_{t-1,I_{t-1}}
=\ell_{t,I_{t-1}}-c_{t-1}$, and
$R\triangleq \E[\sum_{t=2}^Tr_t^2]$. Also, for notational convenience, we define $\varphi(z)\triangleq z-z^2/2$ and set
$\phi_t\triangleq (\varphi(\ell_{t,1}),\ldots,\varphi(\ell_{t,K}))$.

\paragraph{Step 1: expand the sampling bias exactly.}
The first lemma replaces the relaxation leading to
\eqref{eq:bubeck-first-order-step} by an identity.  It is a scalar
calculation based only on the explicit difference between $p_t$ and $x_t$.

\begin{lemma}
\label{lem:exact-bias}
Algorithm~\ref{alg:known} satisfies that
\begin{align}
 \E\left[\sum_{t=2}^TB_t\right]
 &=4\eta\E\left[\sum_{t=2}^Tv_t^2\right]
 -4\eta\E\left[\sum_{t=2}^Tr_t^2\right]
 +\alpha\E\left[\sum_{t=2}^T
 \inner{\phi_t}{p_{t-1}-p_t}\right].
 \label{eq:exact-bias}
\end{align}
\end{lemma}

\begin{proof}
Fix any $t\ge2$ and condition on $\mathcal F_{t-1}$.  Since $I_t\sim p_t$, we know that
$\E_t[r_t^2]=\sum_{i=1}^Kp_{t,i}(\ell_{t,i}-c_{t-1})^2$, while
$(\ell_{t,I_{t-1}}-c_{t-1})^2=v_t^2$.  Moreover, the definition of $p_t$ gives
$p_t=(x_t+\alpha(1-c_{t-1})e_{I_{t-1}})/(1+\alpha(1-c_{t-1}))$.
Consequently,
$p_t-x_t=\alpha(1-c_{t-1})(e_{I_{t-1}}-p_t)$, and hence
\begin{equation}
 B_t = \alpha(1-c_{t-1})\sum_{i=1}^Kp_{t,i}
 \bigl(\ell_{t,I_{t-1}}-\ell_{t,i}\bigr).
 \label{eq:conditional-bias}
\end{equation}

We next separate the residual-square difference
$4\eta(v_t^2-\E_t[r_t^2])$ from \eqref{eq:conditional-bias}.  This is the
quantity whose negative part will later be combined with the OMD stability
term $D_t$.  For each $i\in[K]$, add and subtract the corresponding
residual-square difference:
\begin{align}
 &\alpha(1-c_{t-1})(\ell_{t,I_{t-1}}-\ell_{t,i})\notag\\
 &=4\eta\left(v_t^2-(\ell_{t,i}-c_{t-1})^2\right)
 +\left[4\eta\left((\ell_{t,i}-c_{t-1})^2-v_t^2\right)
 +\alpha(1-c_{t-1})(\ell_{t,I_{t-1}}-\ell_{t,i})\right]\notag\\
 &=4\eta\left(v_t^2-(\ell_{t,i}-c_{t-1})^2\right)
 +4\eta(\ell_{t,I_{t-1}}-\ell_{t,i})
 (2-\ell_{t,I_{t-1}}-\ell_{t,i})\notag\\
 &=4\eta\left(v_t^2-(\ell_{t,i}-c_{t-1})^2\right)
 +\alpha\left(\varphi(\ell_{t,I_{t-1}})-\varphi(\ell_{t,i})\right).
 \label{eq:add-subtract-residual-difference}
\end{align}
Multiplying \eqref{eq:add-subtract-residual-difference} by $p_{t,i}$ and
summing over $i\in[K]$ now yields
\begin{equation}
 B_t=4\eta v_t^2-4\eta\E_t\left[r_t^2\right]
 +\alpha\left(\varphi(\ell_{t,I_{t-1}})
 -\E_t\left[\varphi(\ell_{t,I_t})\right]\right).
 \label{eq:one-step-exact-bias}
\end{equation}
It remains to identify the expectation of the final term.  As the loss
vectors $\ell_1,\ldots,\ell_T$ are oblivious, the tower property gives
\begin{align}
 &\E\left[\varphi(\ell_{t,I_{t-1}})
 -\E_t\left[\varphi(\ell_{t,I_t})\right]\right]=\E\left[\inner{\phi_t}{p_{t-1}}-\inner{\phi_t}{p_t}\right]
 =\E\left[\inner{\phi_t}{p_{t-1}-p_t}\right].
 \label{eq:expected-signed-transport}
\end{align}
Taking expectations in \eqref{eq:one-step-exact-bias}, combining
\eqref{eq:expected-signed-transport}, and summing over $t=2,\ldots,T$
proves \eqref{eq:exact-bias}.
\end{proof}

\paragraph{Step 2: control the stability of OMD updates $x_t$.}
We next control the two terms involving the OMD updates.  The first bound in
the following lemma compares the stability term $D_t$ with the negative
residual square in Lemma~\ref{lem:exact-bias}, and the second controls the
movement $x_{t+1}-x_t$ in a weighted Euclidean norm.  Both are standard in
the analysis of OMD; see, for example,
\citet[Definition~11 and Lemmas~12--14]{wei2018adaptive}.  We include their
specialization to our update because the orientation of $D_t$ and the strict
constant in \eqref{eq:known-D-upper} are useful later.  The third bound is not inherited from standard OMD analysis and controls the inner product between $\phi_t$ and the OMD update
$x_t-x_{t+1}$. This turns out to be important to bound the part of the final term in
Lemma~\ref{lem:exact-bias} arising from the movement of $x_t$.

\begin{lemma}[Log-barrier stability]
\label{lem:logbarrier-stability}
Suppose $\eta\le1/162$.  For every $t$, Algorithm~\ref{alg:known} satisfies
\begin{align}
 D_t&\le0.6\eta r_t^2,
 \label{eq:known-D-upper}\\
 \left|\inner{z}{x_{t+1}-x_t}\right|
 &\le 1.5\sqrt{\eta D_t}
 \sqrt{\inner{x_t}{z^{\odot2}}}
 \quad\text{for every }z\in\mathbb R^K,
 \label{eq:weighted-response-main}\\
 \E_t\left[\inner{\phi_t}{x_t-x_{t+1}}\right]
 &\le 8\E_t\left[D_t\right].
 \label{eq:direct-response}
\end{align}
In addition, $x_{t+1,i}/x_{t,i}\in[0.99,1.01]$ for every $i\in[K]$, and
$D_\Psi(x_{t+1},x_t)\le2D_t$.
\end{lemma}

\begin{proof}
Fix the history before round $t$.  For every fixed $i\in[K]$ and
$s\in[0,1]$, let $x^{(i)}(s)$ be the OMD update obtained when the estimator
is $c_{t-1}\mathbf1+s(\ell_{t,i}-c_{t-1})e_i/p_{t,i}$.  Thus
$x^{(i)}(0)=x_t$, $x_{t+1}=x^{(I_t)}(1)$, and
$r_t=\sum_{i=1}^K\ind\{I_t=i\}(\ell_{t,i}-c_{t-1})$.  We first prove that
$x_j^{(i)}(1)/x_{t,j}\in[0.99,1.01]$ for every fixed $i,j\in[K]$.
The common vector
$c_{t-1}\mathbf1$ has zero inner product with every simplex-tangent
direction. Thus $x^{(i)}(1)$ can be equivalently written as
$\argmin_{x\in\Delta_K}\left(\inner{x}{(\ell_{t,i}-c_{t-1})e_i/p_{t,i}}+D_{\Psi}(x,x_t)\right)$.
The dual local norm of $(\ell_{t,i}-c_{t-1})e_i/p_{t,i}$ at $x_t$ is
bounded as
\begin{equation}
 \eta x_{t,i}^2\frac{(\ell_{t,i}-c_{t-1})^2}{p_{t,i}^2}
 \le\eta(1+8\eta)^2(\ell_{t,i}-c_{t-1})^2
 \le\eta(1+8\eta)^2<\frac19,
 \label{eq:known-local-dual-norm}
\end{equation}
where we used $p_{t,i}\ge x_{t,i}/(1+8\eta)$ and
$|\ell_{t,i}-c_{t-1}|\le1$.
We next modify the Dikin calculation in Lemmas~12--14 of
\citet{wei2018adaptive} only by retaining the exact curvature of the
log-barrier. Define the
local displacement associated with the fixed index $i$ by
$\Delta_{t,i}^{\mathrm{loc}}=\left(\sum_{j=1}^K(x_j^{(i)}(1)-x_{t,j})^2/x_{t,j}^2\right)^{1/2}$.
Instead of replacing the Hessian along the update segment by the relaxed
factor $1/3$ in \citet[Eq.~(15)]{wei2018adaptive}, direct integration gives
\begin{align}
 &\inner{\nabla(\eta\Psi)(x^{(i)}(1))-\nabla(\eta\Psi)(x_t)}{x^{(i)}(1)-x_t}\notag\\
 &=\int_0^1\inner{\nabla^2(\eta\Psi)(x_t+s(x^{(i)}(1)-x_t))(x^{(i)}(1)-x_t)}{x^{(i)}(1)-x_t}\,ds\notag\\
 &=\int_0^1\sum_{j=1}^K
 \frac{(x_j^{(i)}(1)-x_{t,j})^2}
 {(x_{t,j}+s(x_j^{(i)}(1)-x_{t,j}))^2}\,ds
 \ge\frac{(\Delta_{t,i}^{\mathrm{loc}})^2}{1+\Delta_{t,i}^{\mathrm{loc}}}.
 \label{eq:known-exact-curvature}
\end{align}
The first equality is the fundamental theorem of calculus applied to
$\nabla(\eta\Psi)$ along the segment from $x_t$ to $x^{(i)}(1)$, and the
second uses
$\nabla^2(\eta\Psi)(x)=\diag(1/x_j^2)$.  To prove the last inequality, the
definition of $\Delta_{t,i}^{\mathrm{loc}}$ gives
$|x_j^{(i)}(1)-x_{t,j}|/x_{t,j}\le\Delta_{t,i}^{\mathrm{loc}}$ for every $j$.
Thus the denominator in the integrand is at most
$x_{t,j}^2(1+s\Delta_{t,i}^{\mathrm{loc}})^2$, and hence the integral is at least
$ \sum_{j=1}^K\frac{(x_j^{(i)}(1)-x_{t,j})^2}{x_{t,j}^2}
 \int_0^1\frac{1}{(1+s\Delta_{t,i}^{\mathrm{loc}})^2}\,ds
 =\frac{(\Delta_{t,i}^{\mathrm{loc}})^2}{1+\Delta_{t,i}^{\mathrm{loc}}}.
$

On the other hand, the KKT equation defining $x^{(i)}(1)$ is
$\nabla\Psi(x^{(i)}(1))-\nabla\Psi(x_t)+((\ell_{t,i}-c_{t-1})/p_{t,i})e_i+\lambda_{t,i}\mathbf1=0$
for a scalar $\lambda_{t,i}$.  Taking its inner product with
$x^{(i)}(1)-x_t$ eliminates the multiplier because
$\inner{\mathbf1}{x^{(i)}(1)-x_t}=0$.  Therefore,
\begin{align}
 &\inner{\nabla(\eta\Psi)(x^{(i)}(1))-\nabla(\eta\Psi)(x_t)}{x^{(i)}(1)-x_t}\notag\\
 &=-\eta\frac{\ell_{t,i}-c_{t-1}}{p_{t,i}}(x_i^{(i)}(1)-x_{t,i})
 \le\eta\frac{|\ell_{t,i}-c_{t-1}|}{p_{t,i}}|x_i^{(i)}(1)-x_{t,i}|
 \le\eta x_{t,i}\frac{|\ell_{t,i}-c_{t-1}|}{p_{t,i}}\Delta_{t,i}^{\mathrm{loc}}.
 \label{eq:known-kkt-curvature}
\end{align}
Combining \eqref{eq:known-exact-curvature} and
\eqref{eq:known-kkt-curvature} gives
\begin{equation}
 \max_{j\in[K]}
 \left|\frac{x_j^{(i)}(1)-x_{t,j}}{x_{t,j}}\right|
 \le\Delta_{t,i}^{\mathrm{loc}}\le
 \frac{\eta x_{t,i}|\ell_{t,i}-c_{t-1}|/p_{t,i}}
 {1-\eta x_{t,i}|\ell_{t,i}-c_{t-1}|/p_{t,i}}
 \le
 \frac{\eta(1+8\eta)}{1-\eta(1+8\eta)}
<0.01.
 \label{eq:known-coordinatewise-stability}
\end{equation}
The first inequality follows from the definition of
$\Delta_{t,i}^{\mathrm{loc}}$.  For the second, if
$\Delta_{t,i}^{\mathrm{loc}}=0$ there is nothing to prove; otherwise, divide
\eqref{eq:known-exact-curvature}--\eqref{eq:known-kkt-curvature} by
$\Delta_{t,i}^{\mathrm{loc}}$ and solve the inequality.  The third
inequality uses $p_{t,i}\ge x_{t,i}/(1+8\eta)$ and
$|\ell_{t,i}-c_{t-1}|\le1$.  The fourth
uses $\eta\le1/162$.  Therefore,
$x_j^{(i)}(1)/x_{t,j}\in[0.99,1.01]$ for every fixed $i,j\in[K]$.
The representation of the realized update above therefore implies
$x_{t+1,j}/x_{t,j}\in[0.99,1.01]$ for every $j\in[K]$.

Next, we prove \eqref{eq:known-D-upper}. We retain the orientation of the Bregman divergence to obtain the strict
constant needed below.  The self-concordant conjugate inequality used in
the proof of \citet[Lemma~14]{wei2018adaptive}, applied with local decrement
$\eta x_{t,i}|\ell_{t,i}-c_{t-1}|/p_{t,i}$, bounds
$\eta D_\Psi(x_t,x^{(i)}(1))$ by $-z-\log(1-z)$ at
$z=\eta x_{t,i}|\ell_{t,i}-c_{t-1}|/p_{t,i}$.  Since
$-z-\log(1-z)\le z^2/(2(1-z))$ for $0\le z<1$,
and $p_{t,i}\ge x_{t,i}/(1+8\eta)$, we obtain that
\begin{equation}
 D_\Psi(x_t,x^{(i)}(1))
 \le \frac{1}{2\eta(1-\eta(1+8\eta))}
 \left(\eta x_{t,i}\frac{|\ell_{t,i}-c_{t-1}|}{p_{t,i}}\right)^2
 \le\frac{(1+8\eta)^2}{2(1-\eta(1+8\eta))}\eta(\ell_{t,i}-c_{t-1})^2.
 \label{eq:known-D-direct}
\end{equation}
For $\eta\le1/162$, the final coefficient is no more than $0.6$, proving
\eqref{eq:known-D-upper} after substituting the realized index $I_t$. 

We next derive the two consequences of coordinatewise stability used below.
By \eqref{eq:known-coordinatewise-stability},
$x_{t+1,i}/x_{t,i}\in[0.99,1.01]$ for every $i\in[K]$.  For every
$z\in[0.99,1.01]$, direct differentiation gives
\begin{align*}
 z-1-\ln z
 &\le2\left(\frac{1}{z}-1+\ln z\right),\\
 \frac{1}{z}-1+\ln z&\ge\frac{(z-1)^2}{2.1}.
\end{align*}
Applying the first inequality coordinatewise gives
\begin{align*}
 D_\Psi(x_{t+1},x_t)
 &=\frac1\eta\sum_{i=1}^K
 \left(\frac{x_{t+1,i}}{x_{t,i}}-1
 -\ln\frac{x_{t+1,i}}{x_{t,i}}\right)\\
 &\le\frac2\eta\sum_{i=1}^K
 \left(\frac{x_{t,i}}{x_{t+1,i}}-1
 +\ln\frac{x_{t+1,i}}{x_{t,i}}\right)\\
 &=2D_\Psi(x_t,x_{t+1})
 =2D_t,
\end{align*}
proving the last argument in the statement. Similarly, the second inequality implies
\begin{align*}
 D_t=\frac1\eta\sum_{i=1}^K
 \left(\frac{x_{t,i}}{x_{t+1,i}}-1
 +\ln\frac{x_{t+1,i}}{x_{t,i}}\right)\ge\frac{1}{2.1\eta}\sum_{i=1}^K
 \left(\frac{x_{t+1,i}-x_{t,i}}{x_{t,i}}\right)^2.
\end{align*}
Rearranging the terms gives the following inequality
\begin{equation}
 \sum_{i=1}^K
 \frac{(x_{t+1,i}-x_{t,i})^2}{x_{t,i}^2}
 \le2.1\eta D_t.
 \label{eq:known-relative-motion}
\end{equation}

Therefore, for every $z\in\mathbb R^K$, applying Cauchy--Schwarz in the
log-barrier local norm gives
\begin{align*}
 \left|\inner{z}{x_{t+1}-x_t}\right|
 \le
 \sqrt{\sum_{j=1}^K x_{t,j}^2z_j^2}
 \sqrt{\sum_{j=1}^K
 \frac{(x_{t+1,j}-x_{t,j})^2}{x_{t,j}^2}}
 \le
 1.5\sqrt{\eta D_t}
 \sqrt{\inner{x_t^{\odot2}}{z^{\odot2}}}
 \le
 1.5\sqrt{\eta D_t}
 \sqrt{\inner{x_t}{z^{\odot2}}},
\end{align*}
where the last inequality uses $x_{t,j}^2\le x_{t,j}$. Therefore, we prove \eqref{eq:weighted-response-main}.

It remains to prove \eqref{eq:direct-response}, which does not follow
from the generic local-norm argument.  We prove it by a direct calculation
from the KKT equation.  Recall that $x^{(i)}(s)$ is defined for each fixed
$i\in[K]$ and $s\in[0,1]$, and write $x_j^{(i)}(s)$ for its $j$th
coordinate.

The KKT equation of $x^{(i)}(s)$ is
$\nabla\Psi(x^{(i)}(s))+s(\ell_{t,i}-c_{t-1})e_i/p_{t,i}
+\lambda^{(i)}(s)\mathbf1
=\nabla\Psi(x_t)$ for a scalar $\lambda^{(i)}(s)$.  Since
$\nabla^2\Psi(x)=\eta^{-1}\diag(1/x_1^2,\ldots,1/x_K^2)$,
differentiating the $j$th coordinate gives
$\frac{d}{ds}x_j^{(i)}(s)=-\eta(x_j^{(i)}(s))^2
((\ell_{t,i}-c_{t-1})\ind(j=i)/p_{t,i}+(\lambda^{(i)})'(s))$.
Summing this identity over $j$ and using
$\sum_{j=1}^K\frac{d}{ds}x_j^{(i)}(s)=0$ determines
$(\lambda^{(i)})'(s)$.  Substitution gives
\begin{align}
 -\frac{d}{ds}x_j^{(i)}(s)
 =\eta\frac{\ell_{t,i}-c_{t-1}}{p_{t,i}}
 \left((x_i^{(i)}(s))^2\ind(j=i)
 -\frac{(x_i^{(i)}(s))^2(x_j^{(i)}(s))^2}
 {\sum_{k=1}^K(x_k^{(i)}(s))^2}\right).
 \label{eq:known-coordinate-response}
\end{align}
Multiplying \eqref{eq:known-coordinate-response} by
$\varphi(\ell_{t,j})$, summing over $j\in[K]$, and integrating over
$s\in[0,1]$ gives
\begin{align}
 \inner{\phi_t}{x_t-x^{(i)}(1)}
 &=-\int_0^1
 \inner{\phi_t}{\frac{d}{ds}x^{(i)}(s)}\,ds\notag\\
 &=\eta\frac{\ell_{t,i}-c_{t-1}}{p_{t,i}}
 \int_0^1\sum_{j=1}^K\varphi(\ell_{t,j})
 \left(
 (x_i^{(i)}(s))^2\ind(j=i)
 -\frac{(x_i^{(i)}(s))^2(x_j^{(i)}(s))^2}
 {\sum_{k=1}^K(x_k^{(i)}(s))^2}
 \right)ds\notag\\
 &=\eta\frac{\ell_{t,i}-c_{t-1}}{p_{t,i}}
 \int_0^1
 \left(
 \varphi(\ell_{t,i})(x_i^{(i)}(s))^2
 -\frac{(x_i^{(i)}(s))^2}
 {\sum_{k=1}^K(x_k^{(i)}(s))^2}
 \sum_{j=1}^K(x_j^{(i)}(s))^2\varphi(\ell_{t,j})
 \right)ds\notag\\
 &=\eta\frac{\ell_{t,i}-c_{t-1}}{p_{t,i}}
 \int_0^1
 \frac{(x_i^{(i)}(s))^2}
 {\sum_{k=1}^K(x_k^{(i)}(s))^2}
 \sum_{j=1}^K(x_j^{(i)}(s))^2
 \left(
 \varphi(\ell_{t,i})-\varphi(\ell_{t,j})
 \right)ds\notag\\
 &=\eta\frac{\ell_{t,i}-c_{t-1}}{p_{t,i}}
 \int_0^1\sum_{j\ne i}
 \frac{(x_i^{(i)}(s))^2(x_j^{(i)}(s))^2}
 {\sum_{k=1}^K(x_k^{(i)}(s))^2}
 \left(
 \varphi(\ell_{t,i})-\varphi(\ell_{t,j})
 \right)ds.
 \label{eq:known-coordinate-response-sum}
\end{align}
We first establish coordinatewise stability along the entire scaled update
path.  Fix $i\in[K]$ and $s\in[0,1]$.  Since $x^{(i)}(s)$ is the OMD
update driven by $s(\ell_{t,i}-c_{t-1})e_i/p_{t,i}$, repeating the
calculation leading to \eqref{eq:known-coordinatewise-stability} gives
\begin{align}
 \max_{j\in[K]}
 \left|
 \frac{x_j^{(i)}(s)-x_{t,j}}{x_{t,j}}
 \right|
 &\le
 \frac{\eta x_{t,i}s|\ell_{t,i}-c_{t-1}|/p_{t,i}}
 {1-\eta x_{t,i}s|\ell_{t,i}-c_{t-1}|/p_{t,i}}\le
 \frac{\eta(1+8\eta)}
 {1-\eta(1+8\eta)}
 <0.01.
 \label{eq:known-scaled-path-stability}
\end{align}
Therefore,
$0.99x_{t,j}\le x_j^{(i)}(s)\le1.01x_{t,j}$ for every $j\in[K]$ and
$s\in[0,1]$.  Consequently,
\begin{align}
 0.94\frac{x_{t,i}^2x_{t,j}^2}{\sum_{k=1}^Kx_{t,k}^2}\leq \frac{0.99^4}{1.01^2}
 \frac{x_{t,i}^2x_{t,j}^2}{\sum_{k=1}^Kx_{t,k}^2}
 &\le
 \frac{(x_i^{(i)}(s))^2(x_j^{(i)}(s))^2}
 {\sum_{k=1}^K(x_k^{(i)}(s))^2}
 \le
 \frac{1.01^4}{0.99^2}
 \frac{x_{t,i}^2x_{t,j}^2}{\sum_{k=1}^Kx_{t,k}^2}\leq 1.07\frac{x_{t,i}^2x_{t,j}^2}{\sum_{k=1}^Kx_{t,k}^2},
 \label{eq:known-branch-weight-comparison}
\end{align}
which implies that
\begin{align}
 \left|
 \frac{(x_i^{(i)}(s))^2(x_j^{(i)}(s))^2}
 {\sum_{k=1}^K(x_k^{(i)}(s))^2}
 -
 \frac{x_{t,i}^2x_{t,j}^2}{\sum_{k=1}^Kx_{t,k}^2}
 \right|
 \le
 0.07
 \frac{x_{t,i}^2x_{t,j}^2}{\sum_{k=1}^Kx_{t,k}^2}.
 \label{eq:known-branch-weight-error}
\end{align}

Since $x_{t+1}=x^{(I_t)}(1)$ and the conditional distribution of $I_t$ is
$p_t$, the definition of conditional expectation gives
$\E_t[\inner{\phi_t}{x_t-x_{t+1}}]
=\sum_{i=1}^Kp_{t,i}\inner{\phi_t}{x_t-x^{(i)}(1)}$.
Substituting \eqref{eq:known-coordinate-response-sum} into this finite sum
cancels each factor $p_{t,i}$ and gives
\begin{align}
 &\E_t\left[\inner{\phi_t}{x_t-x_{t+1}}\right]\notag\\
 &=
 \eta\sum_{1\le i<j\le K}
 \left(\varphi(\ell_{t,i})-\varphi(\ell_{t,j})\right)
 \int_0^1
 \left(
 (\ell_{t,i}-c_{t-1})
 \frac{(x_i^{(i)}(s))^2(x_j^{(i)}(s))^2}
 {\sum_{k=1}^K(x_k^{(i)}(s))^2}
 -
 (\ell_{t,j}-c_{t-1})
 \frac{(x_i^{(j)}(s))^2(x_j^{(j)}(s))^2}
 {\sum_{k=1}^K(x_k^{(j)}(s))^2}
 \right)ds. 
 \label{eq:known-paired-response}
\end{align}
Adding and subtracting
$x_{t,i}^2x_{t,j}^2/\sum_{k=1}^Kx_{t,k}^2$ inside the integral and
applying \eqref{eq:known-branch-weight-error} yield
\begin{align}
 \E_t\left[\inner{\phi_t}{x_t-x_{t+1}}\right]&\le
 \eta\sum_{1\le i<j\le K}
 \frac{x_{t,i}^2x_{t,j}^2}{\sum_{k=1}^Kx_{t,k}^2}
 \big(
 \left(\varphi(\ell_{t,i})-\varphi(\ell_{t,j})\right)
 \left(\ell_{t,i}-\ell_{t,j}\right)
 \notag\\
 &\qquad
 +0.07
 \left|\varphi(\ell_{t,i})-\varphi(\ell_{t,j})\right|
 \left(
 |\ell_{t,i}-c_{t-1}|+|\ell_{t,j}-c_{t-1}|
 \right)
 \big).
 \label{eq:known-paired-response-upper}
\end{align}
Since $\varphi(x)$ is 1-Lipschitz on $[0,1]$ and direct calculation shows that
\begin{align*}
 (\ell_{t,i}-\ell_{t,j})^2
 &\le
 2\left(
 (\ell_{t,i}-c_{t-1})^2+
 (\ell_{t,j}-c_{t-1})^2
 \right),\\
 |\ell_{t,i}-\ell_{t,j}|
 \left(
 |\ell_{t,i}-c_{t-1}|+|\ell_{t,j}-c_{t-1}|
 \right)
 &\le
 2\left(
 (\ell_{t,i}-c_{t-1})^2+
 (\ell_{t,j}-c_{t-1})^2
 \right),
\end{align*}
substituting these inequalities into
\eqref{eq:known-paired-response-upper} gives
\begin{align}
 \E_t\left[\inner{\phi_t}{x_t-x_{t+1}}\right]
 &\le
 2.14\eta\sum_{1\le i<j\le K}
 \frac{x_{t,i}^2x_{t,j}^2}{\sum_{k=1}^Kx_{t,k}^2}
 \left(
 (\ell_{t,i}-c_{t-1})^2+
 (\ell_{t,j}-c_{t-1})^2
 \right)\notag\\
 &=2.14\eta\sum_{i=1}^K
 \left(
 x_{t,i}^2-
 \frac{x_{t,i}^4}{\sum_{j=1}^Kx_{t,j}^2}
 \right)
 (\ell_{t,i}-c_{t-1})^2\notag\\
 &\le
 2.2\eta\sum_{i=1}^K
 \left(
 x_{t,i}^2-
 \frac{x_{t,i}^4}{\sum_{j=1}^Kx_{t,j}^2}
 \right)
 (\ell_{t,i}-c_{t-1})^2.
 \label{eq:known-response-upper}
\end{align}

It remains to compare the right-hand side of
\eqref{eq:known-response-upper} with $D_t$.  For each fixed $i\in[K]$, the
KKT condition defining $x^{(i)}(1)$ gives
\begin{equation*}
 \nabla\Psi(x_t)-\nabla\Psi(x^{(i)}(1))
 =
 c_{t-1}\mathbf 1+
 \frac{\ell_{t,i}-c_{t-1}}{p_{t,i}}e_i+
 \lambda_{t,i}\mathbf 1
\end{equation*}
for some scalar $\lambda_{t,i}$.  By the definition of the Bregman divergence, we know that
\begin{align}
 D_\Psi(x_t,x^{(i)}(1))+D_\Psi(x^{(i)}(1),x_t)
 =
 \inner{\nabla\Psi(x_t)-\nabla\Psi(x^{(i)}(1))}
 {x_t-x^{(i)}(1)}=
 \frac{\ell_{t,i}-c_{t-1}}{p_{t,i}}
 (x_{t,i}-x_i^{(i)}(1)).
 \label{eq:known-symmetric-divergence}
\end{align}

Next, setting $j=i$ in \eqref{eq:known-coordinate-response} and integrating
over $s\in[0,1]$ gives
\begin{align}
 x_{t,i}-x_i^{(i)}(1)
 &=
 \eta\frac{\ell_{t,i}-c_{t-1}}{p_{t,i}}
 \int_0^1
 \left(
 (x_i^{(i)}(s))^2-
 \frac{(x_i^{(i)}(s))^4}
 {\sum_{k=1}^K(x_k^{(i)}(s))^2}
 \right)ds.
 \label{eq:known-sampled-coordinate-movement}
\end{align}
Substituting \eqref{eq:known-sampled-coordinate-movement} into
\eqref{eq:known-symmetric-divergence} yields
\begin{align}
 D_\Psi(x_t,x^{(i)}(1))+D_\Psi(x^{(i)}(1),x_t)
 &=
 \eta\frac{(\ell_{t,i}-c_{t-1})^2}{p_{t,i}^2}
 \int_0^1
 \left(
 (x_i^{(i)}(s))^2-
 \frac{(x_i^{(i)}(s))^4}
 {\sum_{k=1}^K(x_k^{(i)}(s))^2}
 \right)ds \notag\\
 &=\eta\frac{(\ell_{t,i}-c_{t-1})^2}{p_{t,i}^2}
 \int_0^1
 \left(
 \sum_{j\ne i}
 \frac{(x_i^{(i)}(s))^2(x_j^{(i)}(s))^2}
 {\sum_{k=1}^K(x_k^{(i)}(s))^2}
 \right)ds \notag\\
 &\geq \eta\cdot\frac{0.94(\ell_{t,i}-c_{t-1})^2}{p_{t,i}^2}\left(
 x_{t,i}^2-
 \frac{x_{t,i}^4}{\sum_{k=1}^Kx_{t,k}^2}
 \right)
 \label{eq:known-branch-divergence}
\end{align}
where the last inequality uses \eqref{eq:known-branch-weight-comparison}.

Finally, as $D_t=D_\Psi(x_t,x^{(I_t)}(1))$, applying the definition
of conditional expectation to both divergences and then using
\eqref{eq:known-branch-divergence} gives
\begin{align}
 \E_t\left[
 D_t+D_\Psi(x_{t+1},x_t)
 \right]
 &=
 \eta\sum_{i=1}^K
 \frac{(\ell_{t,i}-c_{t-1})^2}{p_{t,i}}
 \int_0^1
 \left(
 (x_i^{(i)}(s))^2-
 \frac{(x_i^{(i)}(s))^4}
 {\sum_{k=1}^K(x_k^{(i)}(s))^2}
 \right)ds\notag\\
 &\ge
 0.94\eta\sum_{i=1}^K
 \frac{(\ell_{t,i}-c_{t-1})^2}{p_{t,i}}
 \left(
 x_{t,i}^2-
 \frac{x_{t,i}^4}{\sum_{k=1}^Kx_{t,k}^2}
 \right)\notag\\
 &\ge
 0.94\eta\sum_{i=1}^K
 \left(
 x_{t,i}^2-
 \frac{x_{t,i}^4}{\sum_{k=1}^Kx_{t,k}^2}
 \right)
 (\ell_{t,i}-c_{t-1})^2.
 \label{eq:known-response-lower}
\end{align}
The last inequality uses
$p_{t,i}\le1$ and its non-negativity: 
$
 x_{t,i}^2-
 \frac{x_{t,i}^4}{\sum_{k=1}^Kx_{t,k}^2}
 =
 \sum_{j\ne i}
 \frac{x_{t,i}^2x_{t,j}^2}{\sum_{k=1}^Kx_{t,k}^2}
 \ge0.
$

Combining \eqref{eq:known-response-upper} and
\eqref{eq:known-response-lower}, and then using
$D_\Psi(x_{t+1},x_t)\le2D_t$, gives
\begin{align*}
 \E_t\left[\inner{\phi_t}{x_t-x_{t+1}}\right]
 &\le
 \frac{2.2}{0.94}
 \E_t\left[D_t+D_\Psi(x_{t+1},x_t)\right]\le
 \frac{6.6}{0.94}\E_t\left[D_t\right]
 \le
 8\E_t\left[D_t\right].
\end{align*}
\end{proof}

We now apply the stability lemma to the part of the final term in
Lemma~\ref{lem:exact-bias} that contains the OMD iterates $x_t$.

\begin{lemma}
\label{lem:omd-iterate-contribution}
Under the conditions of Lemma~\ref{lem:logbarrier-stability}, Algorithm~\ref{alg:known} guarantees that
\begin{equation}
 \E\left[\sum_{t=2}^T
 \inner{\phi_t}{x_{t-1}-x_t}\right]
  \le 8\eta\left(1+R+\sqrt{R\secondorder}\right).
 \label{eq:omd-iterate-contribution}
\end{equation}
\end{lemma}
\begin{proof}
For every $t\ge2$, write
$\phi_t=\phi_{t-1}+(\phi_t-\phi_{t-1})$.  The left-hand side of
\eqref{eq:omd-iterate-contribution} is therefore equal to $S_1+S_2$, where
\begin{align}
 S_1&=\E\left[\sum_{t=2}^T
 \inner{\phi_{t-1}}{x_{t-1}-x_t}\right],
 \label{eq:central-S1}\\
 S_2&=\E\left[\sum_{t=2}^T
 \inner{\phi_t-\phi_{t-1}}{x_{t-1}-x_t}\right].
 \label{eq:central-S2}
\end{align}

The term $S_1$ pairs $\phi_{t-1}$ with the OMD update performed on round
$t-1$.  Applying \eqref{eq:direct-response} on round $t-1$ and summing over
$t=2,\ldots,T$ gives
\begin{align*}
 S_1
 &\le 8\E\left[\sum_{t=2}^TD_{t-1}\right]
 =8\E\left[\sum_{s=1}^{T-1}D_s\right]\le4.8\eta\E\left[\sum_{s=1}^{T-1}r_s^2\right]
 \le4.8\eta(1+R),
\end{align*}
where the second inequality is due to \eqref{eq:known-D-upper} and the last inequality follows from $r_1^2\le1$ and
$R=\E[\sum_{s=2}^Tr_s^2]$.

We next control $S_2$.  We treat its first summand separately and write
\begin{align*}
 S_2
 &=
 \E\left[\inner{\phi_2-\phi_1}{x_1-x_2}\right]
 +\E\left[\sum_{t=3}^T
 \inner{\phi_t-\phi_{t-1}}{x_{t-1}-x_t}\right].
\end{align*}
Applying \eqref{eq:weighted-response-main} on round $1$ with
$z=\phi_2-\phi_1$ gives
\begin{align*}
 \left|
 \E\left[\inner{\phi_2-\phi_1}{x_1-x_2}\right]
 \right|
 &\le
 1.5\E\left[
 \sqrt{\eta D_1}
 \sqrt{\inner{x_1}{(\phi_2-\phi_1)^{\odot2}}}
 \right]\le
 1.5\sqrt{0.6}\eta
 \le1.2\eta.
\end{align*}
Here \eqref{eq:known-D-upper} and $r_1^2\le1$ imply
$D_1\le0.6\eta$, while the one-Lipschitz property of $\varphi$ and
$\ell_1,\ell_2\in[0,1]^K$ imply
$\inner{x_1}{(\phi_2-\phi_1)^{\odot2}}\le1$.

For the remaining summands, applying
\eqref{eq:weighted-response-main} on round $t-1$ and then
Cauchy--Schwarz over time and expectation gives
\begin{align}
 &\left|
 \E\left[\sum_{t=3}^T
 \inner{\phi_t-\phi_{t-1}}{x_{t-1}-x_t}\right]
 \right|\le
 1.5\sqrt{
 \eta\E\left[\sum_{t=3}^TD_{t-1}\right]}
 \sqrt{
 \E\left[\sum_{t=3}^T
 \inner{x_{t-1}}{(\phi_t-\phi_{t-1})^{\odot2}}\right]}.
 \label{eq:central-cauchy}
\end{align}
By \eqref{eq:known-D-upper} and the definition of $R$,
\begin{align*}
 \E\left[\sum_{t=3}^TD_{t-1}\right]
 &=
 \E\left[\sum_{s=2}^{T-1}D_s\right]
 \le0.6\eta\E\left[\sum_{s=2}^{T-1}r_s^2\right]
 \le0.6\eta R.
\end{align*}
Moreover, since $\varphi'(z)=1-z\in[0,1]$ for $z\in[0,1]$,
$\varphi$ is 1-Lipschitz.  Hence
\begin{align*}
 \inner{x_{t-1}}{(\phi_t-\phi_{t-1})^{\odot2}}
=
 \sum_{i=1}^Kx_{t-1,i}
 (\varphi(\ell_{t,i})-\varphi(\ell_{t-1,i}))^2\le
 \|\ell_t-\ell_{t-1}\|_\infty^2,
\end{align*}
and therefore
\begin{align*}
 \E\left[\sum_{t=3}^T
 \inner{x_{t-1}}{(\phi_t-\phi_{t-1})^{\odot2}}\right]
 &\le\secondorder.
\end{align*}
Substituting these two bounds into \eqref{eq:central-cauchy} gives
\begin{align*}
 |S_2|
 &\le
 1.2\eta+1.5\sqrt{0.6}\eta\sqrt{R\secondorder}\le
 1.2\eta\left(1+\sqrt{R\secondorder}\right).
\end{align*}
Combining the bounds on $S_1$ and $S_2$ yields
\begin{align*}
 S_1+S_2
 &\le
 4.8\eta(1+R)
 +1.2\eta\left(1+\sqrt{R\secondorder}\right)\le
8\eta\left(1+R+\sqrt{R\secondorder}\right),
\end{align*}
which proves \eqref{eq:omd-iterate-contribution}.
\end{proof}

\paragraph{Step 3: control the difference between $p_t$ and $x_t$.}
Define $b_t\triangleq p_t-x_t$.  Then
$p_{t-1}-p_t=(x_{t-1}-x_t)+(b_{t-1}-b_t)$.  The first difference is handled
by Lemma~\ref{lem:omd-iterate-contribution}.  It remains to control the part
containing $b_t$, which arises because $p_t$ places additional mass on the
arm selected on the preceding round.

\begin{lemma}
\label{lem:sampling-distribution-correction}
Let $b_t=p_t-x_t$. Algorithm~\ref{alg:known} with $\eta\le1/162$ guarantees that
\begin{equation}
 \E\left[\sum_{t=2}^T
 \inner{\phi_t}{b_{t-1}-b_t}\right]
  \le 40\eta\left(1+\secondorder+\sqrt{R\secondorder}\right).
 \label{eq:sampling-distribution-correction}
\end{equation}
\end{lemma}

\begin{proof}
We first explain why it suffices to control the squared norm of
$\E[b_t]$.  Since the loss table is oblivious, $\phi_t$ is deterministic,
and rearranging the summation terms gives
\begin{align}
 \E\left[\sum_{t=2}^T\inner{\phi_t}{b_{t-1}-b_t}\right]
 &=\inner{\phi_2}{\E\left[b_1\right]}
 -\inner{\phi_T}{\E\left[b_T\right]}+\sum_{t=2}^{T-1}
 \inner{\phi_{t+1}-\phi_t}{\E\left[b_t\right]}.
 \label{eq:recent-Abel}
\end{align}
Moreover, Cauchy--Schwarz inequality and the $1$-Lipschitz
property of $\varphi(\cdot)$ imply that
\begin{align}
 \sum_{t=2}^{T-1}
 \inner{\phi_{t+1}-\phi_t}{\E\left[b_t\right]}&\le
 \left(\sum_{t=2}^{T-1}\|\phi_{t+1}-\phi_t\|_\infty^2\right)^{1/2}
 \left(\sum_{t=2}^{T-1}\|\E\left[b_t\right]\|_1^2\right)^{1/2}\le
 \sqrt{\secondorder}
 \left(\sum_{t=2}^{T-1}\|\E\left[b_t\right]\|_1^2\right)^{1/2}.
 \label{eq:recent-holder-cauchy}
\end{align}
Therefore, it remains only to prove the following
\begin{align}
 \sum_{t=2}^T\|\E\left[b_t\right]\|_1^2
 \le16\alpha^2(1+R+\secondorder).
 \label{eq:recent-square-sum}
\end{align}

For $i\in[K]$, define
$\beta_{t,i}\triangleq \alpha(1-\ell_{t-1,i})/(1+\alpha(1-\ell_{t-1,i}))$.
The sampling rule gives
$b_t=\beta_{t,I_{t-1}}(e_{I_{t-1}}-x_t)$.  Since the loss table is
oblivious, $\beta_{t,i}$ is also deterministic. Therefore, direct calculation shows that for each $i\in[K]$,
\begin{align}
    \E\left[b_{t,i}\right] &= \E\left[p_{t,i}-x_{t,i}\right] \tag{by definition} \\
    &=\E\left[\beta_{t,I_{t-1}}
 \left(\ind\{I_{t-1}=i\}-x_{t,i}\right)\right] \tag{by definition of $p_t$} \\
 &=\beta_{t,i}\E\left[p_{t-1,i}\right] - \E\left[\beta_{t,I_{t-1}}x_{t,i}\right] \tag{since $\beta_{t,i}$ is deterministic} \\
 &=\beta_{t,i}
 \left(\E\left[p_{t-1,i}\right]-\E\left[x_{t,i}\right]\right)
 +\E\left[
 (\beta_{t,i}-\beta_{t,I_{t-1}})x_{t,i}\right] \notag\\
 &=\beta_{t,i}\left(\E\left[b_{t-1,i}\right]
 -\left(\E\left[x_{t,i}\right]
 -\E\left[x_{t-1,i}\right]\right)\right)
 +\E\left[(\beta_{t,i}-\beta_{t,I_{t-1}})x_{t,i}\right],
 \label{eq:recent-recursion}
\end{align}
where the last equality uses the fact that $p_{t-1}=x_{t-1}+b_{t-1}$.
Since $0\le\beta_{t,i}\le\alpha$ for every $i\in[K]$, taking the
$\ell_1$ norm in \eqref{eq:recent-recursion} gives
\begin{align*}
 \|\E\left[b_t\right]\|_1
 &\le\alpha\|\E\left[b_{t-1}\right]\|_1
 +\alpha\|\E\left[x_t\right]-\E\left[x_{t-1}\right]\|_1
 +\left\|\left(\E\left[(\beta_{t,i}-\beta_{t,I_{t-1}})x_{t,i}\right]\right)_{i=1}^K\right\|_1.
\end{align*}
We next bound the final term. For notational convenience, we draw an
auxiliary arm $J$ from $x_t$ conditionally on $\mathcal F_{t-1}$ according to $x_t$. Since the mapping $z\mapsto \alpha(1-z)/(1+\alpha(1-z))$ is $\alpha$-Lipschitz, we have $|\beta_{t,i}-\beta_{t,j}|\le\alpha|\ell_{t-1,i}-\ell_{t-1,j}|$.
The triangle inequality followed by the conditional law of $J$ and
Jensen's inequality therefore gives
\begin{align*}
 \left\|\left(\E\left[(\beta_{t,i}-\beta_{t,I_{t-1}})x_{t,i}\right]\right)_{i=1}^K\right\|_1
 &\le\E\left[\sum_{i=1}^Kx_{t,i}|\beta_{t,i}-\beta_{t,I_{t-1}}|\right]\\
 &=\E\left[|\beta_{t,J}-\beta_{t,I_{t-1}}|\right]\\
 &\le\alpha\E\left[|\ell_{t-1,J}-\ell_{t-1,I_{t-1}}|\right]\\
 &\le\alpha\sqrt{\E\left[(\ell_{t-1,J}-\ell_{t-1,I_{t-1}})^2\right]}.
\end{align*}
To control the last expectation, observe directly from the definition of
$\beta_{t,I_{t-1}}$ that
$p_t=(1-\beta_{t,I_{t-1}})x_t+\beta_{t,I_{t-1}}e_{I_{t-1}}$.
Consequently, conditionally on $\mathcal F_{t-1}$, the draw $I_t\sim p_t$
can be generated using a Bernoulli coin that is independent of $J$
conditionally on $\mathcal F_{t-1}$: set $I_t=I_{t-1}$ with probability
$\beta_{t,I_{t-1}}$, and set $I_t=J$ otherwise.  Under this equivalent
coupling,
\begin{align*}
 \E_t\left[r_t^2\right]
 &=\beta_{t,I_{t-1}}(\ell_{t,I_{t-1}}-\ell_{t-1,I_{t-1}})^2
 +(1-\beta_{t,I_{t-1}})\E\left[(\ell_{t,J}-\ell_{t-1,I_{t-1}})^2\mid\mathcal F_{t-1}\right]\\
 &\ge(1-\alpha)\E\left[(\ell_{t,J}-\ell_{t-1,I_{t-1}})^2\mid\mathcal F_{t-1}\right].
\end{align*}
Taking the full expectation and using
$(a+b)^2\le2a^2+2b^2$ now gives
\begin{align*}
 \left\|\left(\E\left[(\beta_{t,i}-\beta_{t,I_{t-1}})x_{t,i}\right]\right)_{i=1}^K\right\|_1^2 &\le\alpha^2\E\left[(\ell_{t-1,J}-\ell_{t-1,I_{t-1}})^2\right]\\
 &\le\frac{2\alpha^2}{1-\alpha}\E\left[r_t^2\right]
 +2\alpha^2\|\ell_t-\ell_{t-1}\|_\infty^2\\
 &\le3\alpha^2\left(\E\left[r_t^2\right]
 +\|\ell_t-\ell_{t-1}\|_\infty^2\right),
\end{align*}
where the last inequality uses $\alpha\le8/162$.  Summing over
$t=2,\ldots,T$ yields
\begin{align*}
 \sum_{t=2}^T\left\|\left(\E\left[(\beta_{t,i}-\beta_{t,I_{t-1}})x_{t,i}\right]\right)_{i=1}^K\right\|_1^2
 \le3\alpha^2(R+\secondorder).
\end{align*}

It remains to bound the movement of $x_t$ in
\eqref{eq:recent-recursion}.  For every $t\ge2$, Jensen's and Cauchy--Schwarz inequality give
\begin{align*}
 \|\E\left[x_t\right]-\E\left[x_{t-1}\right]\|_1^2
 &\le\E\left[\|x_t-x_{t-1}\|_1^2\right]\\
 &\le\E\left[\sum_{i=1}^K
 \frac{(x_{t,i}-x_{t-1,i})^2}{x_{t-1,i}}\right]\\
 &\le\E\left[\sum_{i=1}^K
 \frac{(x_{t,i}-x_{t-1,i})^2}{x_{t-1,i}^2}\right]
 \le2.1\eta\E\left[D_{t-1}\right],
\end{align*}
where the third inequality uses $x_{t-1,i}\le1$, and the last inequality
is \eqref{eq:known-relative-motion} applied on round $t-1$.  Summing this
display and then applying \eqref{eq:known-D-upper} gives
\begin{align*}
 \sum_{t=2}^T\|\E\left[x_t\right]-\E\left[x_{t-1}\right]\|_1^2
 &\le2.1\eta\E\left[\sum_{t=1}^{T-1}D_t\right]\le1.26\eta^2\E\left[\sum_{t=1}^{T-1}r_t^2\right]\le2\eta^2(1+R).
\end{align*}

Finally, we square the preceding bound on $\|\E[b_t]\|_1$, use
$(a+b+c)^2\le3a^2+3b^2+3c^2$, and sum over $t=2,\ldots,T$.  The two
bounds just proved yield
\begin{align*}
 \sum_{t=2}^T\|\E\left[b_t\right]\|_1^2
 &\le3\alpha^2\sum_{t=2}^T\|\E\left[b_{t-1}\right]\|_1^2
 +6\alpha^2\eta^2(1+R)+9\alpha^2(R+\secondorder).
\end{align*}
Rearranging the terms and using the fact that $6\eta^2<1$, we know that
\begin{align*}
 \left(1-3\alpha^2\right)\sum_{t=2}^T\|\E\left[b_t\right]\|_1^2
 &\le10\alpha^2(1+R+\secondorder),
\end{align*}
and $1-3\alpha^2>99/100$ proves \eqref{eq:recent-square-sum}.

Finally, the endpoint terms in \eqref{eq:recent-Abel} are at most $\alpha$
because $b_1=0$, $\|\E[b_T]\|_1\le2\alpha$, and
$0\le\varphi\le1/2$.  Combining \eqref{eq:recent-holder-cauchy} and
\eqref{eq:recent-square-sum} gives
\begin{align*}
 \E\left[\sum_{t=2}^T\inner{\phi_t}{b_{t-1}-b_t}\right]
 &\le\alpha+4\alpha\sqrt{\secondorder(1+R+\secondorder)}\le5\alpha\left(1+\secondorder+\sqrt{R\secondorder}\right).
\end{align*}
Since $\alpha=8\eta$, this proves
\eqref{eq:sampling-distribution-correction}.
\end{proof}

\subsection{Proof of Theorem~\ref{thm:known}}

\begin{proof}
Fix an arbitrary arm $i\in[K]$ and the corresponding comparator $u^{(i)}$
defined above.  The OMD decomposition
\eqref{eq:known-omd-decomposition} reduces the proof to bounding $D_t+B_t$.  We now combine the three steps above in the same order in which
they were proved.

Lemma~\ref{lem:exact-bias} and the identity
$p_t=x_t+b_t$ give, up to the $\order(1)$ first-round boundary,
\begin{align}
 \E\left[\sum_{t=1}^T(D_t+B_t)\right]
 &=\E\left[\sum_{t=2}^TD_t\right]
 +4\eta\E\left[\sum_{t=2}^Tv_t^2\right]-4\eta R\notag\\
 &\quad+\alpha\E\left[\sum_{t=2}^T
 \inner{\phi_t}{x_{t-1}-x_t}\right]
 +\alpha\E\left[\sum_{t=2}^T
 \inner{\phi_t}{b_{t-1}-b_t}\right]+\order(1).
 \label{eq:theorem-exact-ledger}
\end{align}
We bound the four terms on the right-hand side separately.  First,
Lemma~\ref{lem:logbarrier-stability} gives
$\E[\sum_{t=2}^TD_t]\le0.6\eta R$.  Second,
$v_t^2\le\|\ell_t-\ell_{t-1}\|_\infty^2$ pathwise, and hence
$\E[\sum_{t=2}^Tv_t^2]\le\secondorder$.  Third,
Lemma~\ref{lem:omd-iterate-contribution}, together with $\alpha=8\eta$,
bounds the term containing $x_{t-1}-x_t$ by
$64\eta^2(1+R+\sqrt{R\secondorder})$.  Finally,
Lemma~\ref{lem:sampling-distribution-correction} bounds the term containing
$b_{t-1}-b_t$ by
$320\eta^2(1+\secondorder+\sqrt{R\secondorder})$.

It remains to absorb the terms depending positively on $R$.  Since
$\eta\le1/162$, we have $64\eta^2R\le0.4\eta R$ and
$384\eta^2\sqrt{R\secondorder}\le0.8\eta R+1.8\eta\secondorder$.  Therefore,
the entire positive $R$-dependent contribution is at most
$1.2\eta R$.
Substituting the four bounds into
\eqref{eq:theorem-exact-ledger} therefore gives
\begin{align}
 \E\left[\sum_{t=1}^T(D_t+B_t)\right]
 &\le(0.6-4+1.2)\eta R+\order(\eta\secondorder+1)\le \order(\eta\secondorder+1).
 \label{eq:theorem-ledger-final}
\end{align}
This proves \eqref{eq:known-ledger-target}.  Finally,
$D_\Psi(u^{(i)},x_1)\le \order(K\log(KT)/\eta)$.  Substituting
\eqref{eq:theorem-ledger-final} into \eqref{eq:known-omd-decomposition} and
maximizing over $i\in[K]$ give the fixed-rate inequality
\begin{equation}
 \Reg_T\le \order\left(\frac{K\log(KT)}{\eta}
 +\eta\secondorder+1\right).
 \label{eq:known-fixed-rate}
\end{equation}
Tuning $\eta$ optimally leads to
\eqref{eq:known-tuned-rate}.
\end{proof}

\section{Unknown \texorpdfstring{$\secondorder$}{Q-infinity-2}: Adaptive Tuning Via Bounded-Scale Estimator}
\label{sec:parameter-free}

Section~\ref{sec:known} tunes Algorithm~\ref{alg:known} using the value of
$\secondorder$. To remove this knowledge, a standard approach is the doubling trick: start from an initial guess, run the algorithm with the learning rate
corresponding to that guess, and restart with a larger guess once the cumulative variation exceeds the current threshold. Such restart constructions are common in adaptive bandit algorithms (e.g., \citet{wei2018adaptive,lee2020graph}). The difficulty here is that the second-order path-length
$\secondorder$ is not observed under bandit feedback.

Recall that $r_t\triangleq c_t-c_{t-1}$ and
$v_t\triangleq\ell_{t,I_{t-1}}-\ell_{t-1,I_{t-1}}$. The quantity $v_t^2$ is
observed whenever $I_t=I_{t-1}$ because $r_t=v_t$ on this event. This suggests
the following single-round path-length estimator
\begin{equation}
 \widehat v_t^2\triangleq
 \frac{\ind(I_t=I_{t-1})r_t^2}{p_{t,I_{t}}}.
 \label{eq:adaptive-v-estimator}
\end{equation}
A direct calculation gives $\E_t[\widehat v_t^2]=v_t^2$. It is therefore natural to use the cumulative sum of $\widehat v_t^2$ to estimate the total second-order path length. Under
Algorithm~\ref{alg:known}, however, the denominator in
\eqref{eq:adaptive-v-estimator} can be arbitrarily small. Indeed, if
$c_{t-1}=1$, then the additional mass on $I_{t-1}$ vanishes and, on the event
$I_t=I_{t-1}$, $p_{t,I_t}=x_{t,I_{t-1}}$. Although the log-barrier guarantees that this
probability is positive, it does not provide a uniform lower bound. Consequently,
$\widehat v_t^2$ can be arbitrarily large, and the cumulative estimate can
overshoot a doubling threshold by an uncontrolled amount.

We therefore make one modification to Algorithm~\ref{alg:known}: we replace
$\lambda_t=\alpha(1-c_{t-1})$ by
$\lambda_t=\alpha(2-c_{t-1})$. Equivalently, we add an $\alpha$ floor to the additional mass assigned to $I_{t-1}$, thereby keeping the denominator bounded away from zero. The loss estimator $\widehat\ell_t$ and the OMD update
remain unchanged. Since $c_{t-1}\in[0,1]$, this modification gives
$p_{t,I_{t-1}}\ge\alpha/(1+\alpha)$ and makes
\eqref{eq:adaptive-v-estimator} uniformly bounded. We later show that this
uniform bound controls the overshoot of the cumulative estimator at every
phase transition. The following lemma records
the two properties of the estimator that are needed by the doubling argument.

\begin{lemma}
\label{lem:adaptive-estimator}
Suppose
$p_t=(x_t+\lambda_te_{I_{t-1}})/(1+\lambda_t)$ with
$\lambda_t=\alpha(2-c_{t-1})$. Then the estimator in
\eqref{eq:adaptive-v-estimator} satisfies
\begin{equation}
 \E_t\left[\widehat v_t^2\right]=v_t^2
 \quad\text{and}\quad
 0\le\widehat v_t^2\le\frac{1+\alpha}{\alpha}.
 \label{eq:adaptive-estimator-properties}
\end{equation}
\end{lemma}

\begin{proof}
The definition of $p_t$ gives
$p_{t,I_{t-1}}\ge\lambda_t/(1+\lambda_t)\ge\alpha/(1+\alpha)$. Moreover,
$r_t=v_t$ on the event $I_t=I_{t-1}$. Therefore,
\begin{align*}
 \E_t\left[\widehat v_t^2\right]
 &=\sum_{i=1}^Kp_{t,i}
 \frac{\ind(i=I_{t-1})(\ell_{t,i}-c_{t-1})^2}{p_{t,i}}
 =v_t^2.
\end{align*}
For the upper bound, the numerator is nonzero only when $I_t=I_{t-1}$, in
which case $p_{t,I_t}=p_{t,I_{t-1}}\ge\alpha/(1+\alpha)$. The result now
follows from $|r_t|\le1$.
\end{proof}

We are now ready to present our adaptive algorithm. Define
$A_K=64K\log(KT)$ and $\eta_0=10^{-5}$. The algorithm starts with
the threshold $H_0=4A_K/\eta_0^2$. During phase $j$, it uses
$\eta_j\triangleq2^{-j}\eta_0$, $\alpha_j\triangleq8\eta_j$, and
$H_j\triangleq4A_K/\eta_j^2$. It samples from $p_t$ and updates $x_{t+1}$
exactly as in Algorithm~\ref{alg:known}, except for the $\alpha_j$ floor in
Line~\ref{line:adaptive-sampler}. After every noninitial round of the phase,
it adds $\widehat v_t^2$ to its estimate $\widehat H_j$. If
$\widehat H_j$ exceeds $H_j$, phase $j$ ends and the outer loop begins phase
$j+1$ on the next round with a uniform initial distribution. The scalar
$\widehat H_j$ is initialized to zero when phase $j$ begins. The following theorem shows that
Algorithm~\ref{alg:adaptive} achieves $\otil(K+\sqrt{KQ_{\infty,2}})$ without
knowing $Q_{\infty,2}$. The overhead is a $\sqrt{\log T}$ factor compared to
the guarantee in Theorem~\ref{thm:known}.

\begin{algorithm}[H]
\caption{Adaptive tuning without knowledge of $\secondorder$}
\label{alg:adaptive}
\KwInput{$K>0,T>0$}
Set $A_K=64K\log(KT)$, $\eta_0=10^{-5}$,
$j=0$, $s_0=t=1$, and $c_0=0$\;
\While{$t\le T$}{
  Set $\eta_j=2^{-j}\eta_0$, $\alpha_j=8\eta_j$,
  $H_j=4A_K/\eta_j^2$, and $\widehat H_j=0$\;
  Phase $j$ starts at round $s_j=t$\;
  Set $x_t=p_t=\frac{1}{K}\mathbf1$, and
  $\widehat v_{s_j}^2=0$\;
  \While{$t\le T$}{
    Draw $I_t\sim p_t$ and observe $c_t=\ell_{t,I_t}$\;
    Set $\widehat\ell_{t,i}=c_{t-1}+
    \ind(I_t=i)(c_t-c_{t-1})/p_{t,i}$ for every $i\in[K]$\;
    Set $x_{t+1}=\argmin_{x\in\Delta_K}
    \left(\inner{x}{\widehat\ell_t}+D_{\Psi_j}(x,x_t)\right)$, where $\Psi_j(x)=\eta_j^{-1}\sum_{i=1}^K\log(1/x_i)$\;\label{line:omd-para-free}
    \If{$t>s_j$}{
      Set $r_t=c_t-c_{t-1}$,
      $\widehat v_t^2=\ind(I_t=I_{t-1})r_t^2/p_{t,I_t}$, and
      $\widehat H_j\leftarrow\widehat H_j+\widehat v_t^2$\;
    }
    \If{$\widehat H_j>H_j$ and $t<T$}{
      Set $t\leftarrow t+1$, $j\leftarrow j+1$\;
      
      \textbf{break}\;
    }
    \ElseIf{$t<T$}{
      Set $\lambda_{t+1}=\alpha_j(2-c_t)$ and
      $p_{t+1}=(x_{t+1}+\lambda_{t+1}e_{I_t})/
      (1+\lambda_{t+1})$\;\label{line:adaptive-sampler}

    }
    
    Set $t\leftarrow t+1$\;
    
  }
}
\end{algorithm}

\begin{theorem}
\label{thm:unknown}
For every $K\ge2$, $T\ge2$, and oblivious loss sequence
$\ell_1,\ldots,\ell_T\in[0,1]^K$, Algorithm~\ref{alg:adaptive} requires no
knowledge of $\secondorder$ and guarantees that
\begin{equation}
 \Reg_T\le \order\left(K\log(KT)+
 \sqrt{K\secondorder\log(KT)\log T}\right).
 \label{eq:unknown-rate}
\end{equation}
\end{theorem}

\subsection{Analysis of Algorithm~\ref{alg:adaptive}}

We first relate the cumulative estimate $\widehat H_j$ to the two squared quantities in
the analysis of Section~\ref{sec:known}. Let $\mathsf J_t$ denote the value of
the algorithm's phase counter immediately before $I_t$ is drawn. Thus
$\mathsf J_t$ is $\mathcal F_{t-1}$-measurable. In the analysis below, $j$ is
always a deterministic phase label; the randomness of the active phase is
carried by $\mathsf J_t$. Consistently with Algorithm~\ref{alg:adaptive},
define the first round of phase $j$ by
$s_j\triangleq\inf\{t\in[T]:\mathsf J_t=j\}$, with the convention
$\inf\varnothing=\infty$, and define
$\chi_{j,t}\triangleq\ind(\mathsf J_t=j,\ t>s_j)$. Hence $\chi_{j,t}$ is the
indicator that round $t$ is a noninitial round of phase $j$, and it is
$\mathcal F_{t-1}$-measurable. Let
$\rho_j\triangleq\Pr(s_j\le T)$ be the probability that phase $j$ is reached.
For notational convenience, we define the following two quantities:
\begin{align}
 R_j&\triangleq\E\left[\sum_{t=2}^T\chi_{j,t}r_t^2\right],
 \label{eq:adaptive-phase-residual}\\
 P_j&\triangleq\E\left[\sum_{t=2}^T\chi_{j,t}v_t^2\right].
 \label{eq:adaptive-phase-path-length}
\end{align}
Specifically, the quantity $R_j$ is the expected cumulative residual square over the noninitial rounds of phase $j$ and the quantity $P_j$ is the
expected cumulative squared loss movement along the previously selected arm. By the unbiasedness of $\widehat v_t^2$, it is also the expected value of the doubling statistic accumulated during phase $j$. The next lemma shows that the doubling statistic has expectation $P_j$,
has controlled overshoot, and that the sum of these phase-wise expectations
is at most the true squared path-length.

\begin{lemma}
\label{lem:adaptive-phase-estimate}
For every deterministic phase label $j\in\{0,\ldots,T-1\}$,
Algorithm~\ref{alg:adaptive} guarantees that
$H_j\rho_{j+1}\le P_j\le2H_j\rho_j$ and
$\sum_{j=0}^{T-1}P_j\le\secondorder$.
\end{lemma}

\begin{proof}
For a deterministic $j$, let
$\widehat H_j^{\mathrm{fin}}\triangleq
\sum_{t=2}^T\chi_{j,t}\widehat v_t^2$ denote the final value accumulated
by phase $j$; it is zero on the event $s_j=\infty$. Since $\chi_{j,t}$ is
$\mathcal F_{t-1}$-measurable, Lemma~\ref{lem:adaptive-estimator} and the tower
property give
$\E[\chi_{j,t}\widehat v_t^2]=\E[\chi_{j,t}v_t^2]$. Summing over the deterministic
set $t=2,\ldots,T$ therefore yields
$\E[\widehat H_j^{\mathrm{fin}}]=P_j$.

Pathwise, reaching phase $j+1$ requires phase $j$ to cross its threshold, so
$H_j\ind(s_{j+1}\le T)\le\widehat H_j^{\mathrm{fin}}$. Before the triggering
round the accumulated value is at most $H_j$, while
Lemma~\ref{lem:adaptive-estimator} bounds the triggering increment by
$(1+\alpha_j)/\alpha_j$. Since $\alpha_j=8\eta_j$ and
$\eta_j\le10^{-5}$, this increment is at most $1/(4\eta_j)$ as

$$
 \frac{1+\alpha_j}{\alpha_j}
 =\frac{1}{8\eta_j}+1
 \le\frac{1}{4\eta_j}.
$$
Moreover,
$H_j=4A_K/\eta_j^2\ge1/(4\eta_j)$. Consequently,
$\widehat H_j^{\mathrm{fin}}\le2H_j\ind(s_j\le T)$ pathwise. Taking expectations in these two pathwise inequalities proves
$H_j\rho_{j+1}\le P_j\le2H_j\rho_j$.

Finally, for every $t$, at most one
indicator $\chi_{j,t}$ equals one. Since the loss table is fixed and
$v_t^2\le\lVert\ell_t-\ell_{t-1}\rVert_\infty^2$, we have
\begin{align*}
 \sum_{j=0}^{T-1}P_j
 &=\E\left[\sum_{t=2}^T\sum_{j=0}^{T-1}\chi_{j,t}v_t^2\right]
 \le\sum_{t=2}^T\lVert\ell_t-\ell_{t-1}\rVert_\infty^2
 =\secondorder.
\end{align*}
This proves the second claim.
\end{proof}

We next show in the following lemma how to control the contribution of the noninitial rounds of each phase. Specifically, as shown
in \eqref{eq:known-omd-decomposition}, in order to bound the regret, the relevant quantity after the
standard OMD decomposition is the sum of $D_t+B_t$, where we recall that $B_t\triangleq \inner{p_t-x_t}{\ell_t}$ and with a slight abuse of notation, we define $D_t\triangleq D_{\Psi_{\mathsf J_t}}(x_t,x_{t+1})$. Here and throughout the phasewise analysis, if round $t$ ends a
phase, $x_{t+1}$ denotes the OMD iterate computed in Line~\ref{line:omd-para-free} before the subsequent reset. The initial round of
each phase is excluded from $\chi_{j,t}$ because its sampling distribution is reset to the uniform distribution and we do not count the path-length quantity across phase boundaries. We will handle its contribution separately in the proof of Theorem~\ref{thm:unknown}. 

\begin{lemma}
\label{lem:adaptive-phase-bound}
For every deterministic phase label $j\in\{0,\ldots,T-1\}$,
\begin{equation}
 \E\left[\sum_{t=2}^T\chi_{j,t}(D_t+B_t)\right]
 \le\order(\eta_jP_j+\rho_j).
 \label{eq:adaptive-phase-bound}
\end{equation}
\end{lemma}

\begin{proof}
Fix a phase label $j\in\{0,1,\dots, T-1\}$. The additional $\alpha_j$ in the sampling weight
$\lambda_t=\alpha_j(2-c_{t-1})$ changes the scalar function used in
Lemma~\ref{lem:exact-bias}. Therefore, we define
$\varphi_2(z)\triangleq2z-z^2/2$, so that
$\varphi_2(z)=\varphi(z)+z$, and set
\begin{equation*}
 \phi_{2,t}\triangleq
 \left(\varphi_2(\ell_{t,1}),\ldots,\varphi_2(\ell_{t,K})\right).
\end{equation*}
We decompose $B_t$ at every noninitial round of phase $j$ by following the
same add-and-subtract calculation as in Lemma~\ref{lem:exact-bias}. Specifically, the definition of the sampling distribution gives
\begin{align}
 B_t
 &=\inner{p_t-x_t}{\ell_t}
 =\alpha_j(2-c_{t-1})\sum_{i=1}^Kp_{t,i}
 \left(\ell_{t,I_{t-1}}-\ell_{t,i}\right).
 \label{eq:adaptive-conditional-bias}
\end{align}
Direct calculation shows that 
 \begin{align}
 &4\eta_j\left((\ell_{t,i}-c_{t-1})^2-(\ell_{t,I_{t-1}}-c_{t-1})^2\right)
 +\alpha_j(2-c_{t-1})(\ell_{t,I_{t-1}}-\ell_{t,i})\notag\\
 &=8\eta_j(\ell_{t,I_{t-1}}-\ell_{t,i})\left(2-\frac{\ell_{t,I_{t-1}}+\ell_{t,i}}{2}\right)
 =\alpha_j\left(\varphi_2(\ell_{t,I_{t-1}})-\varphi_2(\ell_{t,i})\right).
 \label{eq:adaptive-scalar-identity}
\end{align}

Multiplying \eqref{eq:adaptive-scalar-identity} by $p_{t,i}$, summing
over $i\in[K]$, and using the fact that $\E_t\left[r_t^2\right]
 =\sum_{i=1}^Kp_{t,i}(\ell_{t,i}-c_{t-1})^2$ and $v_t^2=(\ell_{t,I_{t-1}}-c_{t-1})^2$ yield
\begin{align}
B_t
 &=4\eta_jv_t^2-4\eta_j\E_t\left[r_t^2\right]
 +\alpha_j\left(\varphi_2(\ell_{t,I_{t-1}})
 -\E_t\left[\varphi_2(\ell_{t,I_t})\right]\right).
 \label{eq:adaptive-exact-bias}
\end{align}
Since $\chi_{j,t}$ is $\mathcal F_{t-1}$-measurable, the tower property
gives $\E\left[\chi_{j,t}
 \E_t\left[\varphi_2(\ell_{t,I_t})\right]\right]
 =\E\left[\chi_{j,t}\varphi_2(\ell_{t,I_t})\right]$.
It remains to bound
\begin{equation*}
 \alpha_j\E\left[\sum_{t=2}^T\chi_{j,t}
 \left(\varphi_2(\ell_{t,I_{t-1}})
 -\varphi_2(\ell_{t,I_t})\right)\right].
\end{equation*}
Two deterministic-horizon steps from Section~\ref{sec:known} cannot be
obtained by simply multiplying each summand by $\chi_{j,t}$. First,
$\chi_{j,t}$ is $\mathcal F_{t-1}$-measurable but need not be
$\mathcal F_{t-2}$-measurable. Although
\[
\E[\ind\{I_{t-1}=i\}\mid\mathcal F_{t-2}]=p_{t-1,i},
\]
the indicator $\chi_{j,t}$ cannot in general be taken outside this
conditional expectation. Consequently, in general, $\E[\chi_{j,t}\ind\{I_{t-1}=i\}]
 \ne \E[\chi_{j,t}p_{t-1,i}].$

The second issue concerns the phasewise tuning of $\eta_j$ using the
path-length estimator. With
$d_{t,i}=\ell_{t,i}-\ell_{t-1,i}$, define
\begin{equation}
 Q_j^{\max}\triangleq
 \E\left[\sum_{t=2}^T
 \chi_{j,t}\max_{i\in[K]}d_{t,i}^2\right].
 \label{eq:adaptive-phase-max-energy}
\end{equation}
The phase indicators are disjoint, and hence
$\sum_jQ_j^{\max}\le\secondorder$. Even after handling the first issue
through stopping-boundary corrections, a direct phasewise application of
the maximum-norm argument from Section~\ref{sec:known} would give a bound
of the form
\begin{equation}
 \E\left[\sum_{t=2}^T\chi_{j,t}(D_t+B_t)\right]
 \le
 \order\left(
 \eta_jP_j+\eta_j^2Q_j^{\max}+\rho_j
 \right).
 \label{eq:adaptive-max-energy-ledger}
\end{equation}
Summing the additional term over phases yields only
$
 \sum_j\eta_j^2Q_j^{\max}
 \le
 \eta_0^2\sum_jQ_j^{\max}
 \le
 \eta_0^2\secondorder,
$
which depends linearly on $\secondorder$ and is therefore insufficient for
the desired bound. The phasewise estimator, however, controls
\begin{equation}
 P_j
 \triangleq
 \sum_{t=2}^T\sum_{i=1}^K
 \E[\chi_{j,t}\ind\{I_{t-1}=i\}]d_{t,i}^2.
 \label{eq:adaptive-sampled-energy}
\end{equation}
Therefore, to exploit the phasewise choice of
$\eta_j$, the analysis must replace $Q_j^{\max}$ by $P_j$.

The first issue can be handled by a one-round correction.  If
$A_{j,t}=\ind\{\mathsf J_t=j\}$, then
\begin{equation}
 A_{j,t-1}=\chi_{j,t}+\delta_{j,t},
 \qquad
 \E\left[\sum_{t=2}^T\delta_{j,t}\right]\le\rho_j,
 \label{eq:adaptive-predictable-extension}
\end{equation}
where $A_{j,t-1}$ is measurable before $I_{t-1}$ is drawn.  At the boundary
we use the OMD update computed before the subsequent reset.  This extension
is sufficient for the OMD part of Lemma~\ref{lem:omd-iterate-contribution}:
if $g_t=\phi_{2,t}-\phi_{2,t-1}$, then
\begin{equation}
 \E\left[\sum_{t=2}^TA_{j,t-1}
 \inner{x_{t-1}}{g_t^{\odot2}}\right]
 \le5(P_j+\rho_j).
 \label{eq:adaptive-omd-local-energy}
\end{equation}
Here we use $|g_{t,i}|\le2|d_{t,i}|$ by the 2-Lipschitzness of $\varphi_2(\cdot)$ and
$p_{t-1,i}\ge x_{t-1,i}/(1+2\alpha_j)$.

The second issue
requires a self-normalized replacement for the unavailable
maximum-coordinate bound.  Following the notation in Section~\ref{sec:known}, we set $b_t\triangleq p_t-x_t$.  With
$
 q_{j,t,i}\triangleq \E[\chi_{j,t}
 (\ind\{I_t=i\}+x_{t,i})],
$
Cauchy--Schwarz inequality gives that
\begin{align}
 &\left|\sum_{t=2}^{T-1}\sum_{i=1}^K
 g_{t+1,i}\E[\chi_{j,t}b_{t,i}]\right|\le
 \left(\sum_{t,i}
 \frac{\bigl(\E[\chi_{j,t}b_{t,i}]\bigr)^2}{q_{j,t,i}}
 \right)^{1/2}
 \left(\sum_{t,i}q_{j,t,i}g_{t+1,i}^2\right)^{1/2}.
 \label{eq:adaptive-self-normalized-obstruction}
\end{align}
With a more careful calculation, we bound the right-hand side terms as follows
\begin{align}
 \sum_{t,i}\frac{\bigl(\E[\chi_{j,t}b_{t,i}]\bigr)^2}{q_{j,t,i}}
 &\le\order\left(\eta_j^2R_j+\alpha_j^2P_j
 +\alpha_j\rho_j\right),\notag\\
 \sum_{t,i}q_{j,t,i}g_{t+1,i}^2
 &\le\order\left(\eta_j^2R_j+P_j+\rho_j\right).
 \label{eq:adaptive-matched-energy-bounds}
\end{align}
The rest of the proof uses
$p_t=x_t+b_t$ in the same order as
Lemmas~\ref{lem:omd-iterate-contribution} and
\ref{lem:sampling-distribution-correction}.  It gives, for every $j$,
\begin{align}
 &\alpha_j\E\left[\sum_{t=2}^T\chi_{j,t}
 \left(\varphi_2(\ell_{t,I_{t-1}})
 -\varphi_2(\ell_{t,I_t})\right)\right]
 \le2000\eta_j^2R_j+\order\left(\eta_j^2P_j+\rho_j\right).
 \label{eq:adaptive-phase-sampling-term}
\end{align}
The full proof of \eqref{eq:adaptive-phase-sampling-term} is deferred to Appendix~\ref{app:phasewise}.
The $\order(\rho_j)$ additional term bounds the initialization term and the
single stopping-boundary term of phase $j$.

It remains to control $D_t$. Since
$\lambda_t=\alpha_j(2-c_{t-1})\le2\alpha_j$, the modified sampling rule
satisfies $p_{t,i}\ge x_{t,i}/(1+2\alpha_j)$. Applying the direct calculation
in \eqref{eq:known-D-direct} with this lower bound gives
\begin{equation}
 D_t\le
 \frac{(1+2\alpha_j)^2}
 {2\left(1-\eta_j(1+2\alpha_j)\right)}
 \eta_jr_t^2
 \le0.6\eta_jr_t^2,
 \label{eq:adaptive-D-upper}
\end{equation}
where the last inequality follows from $\alpha_j=8\eta_j$ and
$\eta_j\le10^{-5}$.

Multiplying \eqref{eq:adaptive-exact-bias} by $\chi_{j,t}$, taking
expectations, and summing over $t$ now give
\begin{align*}
 \E\left[\sum_{t=2}^T\chi_{j,t}(D_t+B_t)\right]
 &\le\left(-3.4\eta_j+2000\eta_j^2\right)R_j
 +\order\left(\eta_jP_j+\rho_j\right).
\end{align*}
Since $2000\eta_j\le0.02$, the coefficient of $R_j$ is negative. Dropping
this term proves \eqref{eq:adaptive-phase-bound}, which is exactly the
phasewise counterpart of \eqref{eq:theorem-ledger-final}.
\end{proof}

We are now ready to prove Theorem~\ref{thm:unknown}.

\begin{proof}[Proof of Theorem~\ref{thm:unknown}]
Fix an arbitrary arm $i\in[K]$. On every realization for which phase $j$
is reached, its rounds form a contiguous interval. Sum
\eqref{eq:omd-one-step-preview} over this interval and compare with the
smoothed version of $e_i$ used in Section~\ref{sec:known}. The OMD
divergences telescope pathwise. Since the phase starts from the uniform
distribution, the comparator term is at most $A_K/\eta_j$ on the event
that the phase is reached. Moreover, $p_{s_j}=x_{s_j}$, so the initial
round has no sampling bias, and its remaining OMD and smoothing terms are
at most $3$. Lemma~\ref{lem:adaptive-phase-bound} therefore
gives
\begin{align}
\E\left[\sum_{t=1}^T\ind(\mathsf J_t=j)
 (\ell_{t,I_t}-\ell_{t,i})\right]\le
 \underbrace{\frac{A_K\rho_j}{\eta_j}}_{\text{comparator term}}
 +\underbrace{\E\left[\sum_{t=2}^T
 \chi_{j,t}(D_t+B_t)\right]}_{\text{noninitial rounds}}
 +3\rho_j
 \le\frac{A_K\rho_j}{\eta_j}
 +\order\left(\eta_jP_j+\rho_j\right).
 \label{eq:adaptive-phase-regret}
\end{align}
It remains to sum this bound over the phases. From
Lemma~\ref{lem:adaptive-phase-estimate} and
$H_j=4A_K/\eta_j^2$, we know that
\begin{equation}
 \eta_jP_j\le\sqrt{8A_K\rho_jP_j}.
 \label{eq:adaptive-movement-charge}
\end{equation}
The lower bound in the same lemma also gives
$4A_K\rho_{j+1}/\eta_j^2\le P_j$. Since
$\eta_{j+1}=\eta_j/2$, we obtain
$A_K\rho_{j+1}/\eta_{j+1}\le\eta_jP_j/2$.
Thus the movement term of phase $j$ also controls the comparator term of
phase $j+1$.

Next, we show that the number of phases is logarithmic. Indeed,
Lemma~\ref{lem:adaptive-estimator} gives
$\widehat v_t^2\le1/(4\eta_j)$ during phase $j$. A completed phase must
cross $H_j=4A_K/\eta_j^2$ and hence contains at least
$16A_K/\eta_j$ noninitial rounds. Let $M$ denote the random number of
phases reached. Pathwise, summing this lower bound over the first $M-1$
completed phases gives
$T\ge16A_K(2^{M-1}-1)/\eta_0$, and therefore
$M\le2\log_2(2T)$. Since
$M=\sum_{j=0}^{T-1}\ind(s_j\le T)$, taking expectations yields
\begin{equation}
 \sum_{j=0}^{T-1}\rho_j
 =\E[M]\le2\log_2(2T).
 \label{eq:adaptive-reached-phase-count}
\end{equation}

Finally, the phase indicators partition the horizon. Summing
\eqref{eq:adaptive-phase-regret}, bounding every comparator term except the
first one by the preceding-phase movement term as above, and using
\eqref{eq:adaptive-movement-charge} yield
\begin{align*}
 \Reg_T(i)
 &\le\frac{A_K}{\eta_0}
 +\order\left(\sum_{j=0}^{T-1}
 \sqrt{A_K\rho_jP_j}+\sum_{j=0}^{T-1}\rho_j\right)\\
 &\le\frac{A_K}{\eta_0}
 +\order\left(\sqrt{A_K\log(T)
 \sum_{j=0}^{T-1}P_j}+\log T\right)\\
 &\le\order\left(K\log(KT)+
 \sqrt{K\secondorder\log(KT)\log T}\right),
\end{align*}
where the second inequality uses Cauchy--Schwarz and the last one uses
Lemma~\ref{lem:adaptive-phase-estimate}, $A_K=64K\log(KT)$, and
$\eta_0=10^{-5}$. The bound holds for
every $i\in[K]$. Taking the maximum over $i$ proves
\eqref{eq:unknown-rate}.
\end{proof}

\section{Conclusion}

We study second-order path-length guarantees for adversarial multi-armed bandits under ordinary bandit feedback. Our main result shows that, when $Q_{\infty,2}$ is known, the recent-arm optimistic log-barrier algorithm of \citet{bubeck2019path} already achieves $O\left(K\log(KT)+\sqrt{K\log(KT)(1+Q_{\infty,2})}\right)$ expected regret against oblivious loss sequences, but requires a more refined analysis. We further remove the knowledge of $Q_{\infty,2}$ through an adaptive restart scheme. Together with the $\Omega(\sqrt{KQ_{\infty,2}})$ lower bound, these results establish the optimal dependence on the second-order path length up to logarithmic factors.

Several questions remain open. First, the additional $\sqrt{\log T}$ factor in the adaptive guarantee arises from the phase-based tuning argument, and it would be interesting to determine whether the oracle rate can be achieved without knowing $Q_{\infty,2}$. Second, our analysis relies essentially on the loss sequence being oblivious. Extending the second-order guarantee to adaptive adversaries would require a different treatment of the signed difference terms used in the proof. More broadly, it would be interesting to understand whether similar second-order path-length guarantees can be obtained in richer partial-information models, such as linear bandits or MDPs.

\paragraph{Acknowledgments.}
The authors used GPT-5.6 for assistance with writing and to  explore proof strategies. The authors have checked the arguments and take full responsibility for the content of the paper.

\bibliographystyle{plainnat}
\bibliography{references}

\newpage

\appendix
\section{Omitted Details in Section~\ref{sec:parameter-free}}
\label{app:phasewise}

In this section, we show the omitted details in Section~\ref{sec:parameter-free}. Specifically, we provide the proof of \eqref{eq:adaptive-phase-sampling-term}, which is the key to control the regret in each phase $j$. For notational convenience, since we will consider each phase index $j$ separately, in the proofs in this section, we will fix a deterministic phase label
$j\in\{0,\ldots,T-1\}$ and abbreviate
\begin{equation*}
 \begin{aligned}
 \chi_t&=\chi_{j,t},\qquad \rho=\rho_j,
 \qquad \eta=\eta_j,\qquad \alpha=\alpha_j,
 \qquad R=R_j,\qquad P=P_j,\\
 s&=s_j,\qquad
 \tau=\max\{u\in[T]:\mathsf J_u=j\},
 \qquad \bar\chi_t=\ind\{\mathsf J_{t-1}=j\},
 \qquad \delta_t=\bar\chi_t-\chi_t.
 \end{aligned}
\end{equation*}
Here $\tau=0$ if phase $j$ is not reached. We also recall that $\phi_{2,t,i}=\varphi_2(\ell_{t,i})$, $b_t=p_t-x_t$, $\alpha=8\eta$, and $\eta\le10^{-5}$. For the theorem statement, we still keep the index $j$.
The following lemma shows the bound on \eqref{eq:adaptive-phase-sampling-term}.

\begin{lemma}
\label{lem:phase-localization}
For every deterministic phase label $j\in\{0,\ldots,T-1\}$,
Algorithm~\ref{alg:adaptive} satisfies
\begin{align}
 &\alpha_j\E\left[\sum_{t=2}^T\chi_{j,t}
 \left(\varphi_2(\ell_{t,I_{t-1}})
 -\varphi_2(\ell_{t,I_t})\right)\right]
 \le2000\eta_j^2R_j+\order\left(\eta_j^2P_j+\rho_j\right).
 \label{eq:appendix-phase-localization}
\end{align}
\end{lemma}

To prove this, following the analysis in Section~\ref{sec:known}, we need to first bound the summation of $\inner{\phi_{2,t}}{b_{t-1}-b_t}$ within each phase. The following lemma shows that this term is bounded by

\begin{lemma}
\label{lem:phasewise-recent-arm}
For every deterministic phase label $j\in\{0,\ldots,T-1\}$,
Algorithm~\ref{alg:adaptive} satisfies
\begin{equation}
 \alpha_j\E\left[\sum_{t=2}^T\chi_{j,t}
 \inner{\phi_{2,t}}{b_{t-1}-b_t}\right]
 \le1700\eta_j^2R_j+2500\eta_j^2P_j+20\rho_j.
 \label{eq:appendix-recent-arm}
\end{equation}
\end{lemma}

\begin{proof}
If phase $j$ is not reached, the left-hand side of
\eqref{eq:appendix-recent-arm} is zero. On the event that it is reached,
$s$ and $\tau$ are its first and last rounds. Since $b_s=0$, the left-hand side can be written as
\begin{equation*}
 \sum_{t=2}^T\chi_{t}
 \inner{\phi_{2,t}}{b_{t-1}-b_t}
 = -\inner{\phi_{2,\tau}}{b_\tau} + 
 \sum_{t=2}^{T-1}\chi_{t+1}
 \inner{\phi_{2,t+1}-\phi_{2,t}}{b_t}.
\end{equation*}
Since $0\le\varphi_2(z)\le3/2$ on $[0,1]$, we have
$\|\phi_{2,\tau}\|_\infty\le3/2$. Moreover, the sampling rule gives
$\|b_\tau\|_1\le4\alpha$. Since phase $j$ is reached with probability
$\rho$, we have
\begin{align*}
 &\alpha\E\left[\sum_{t=2}^T\chi_t
 \inner{\phi_{2,t}}{b_{t-1}-b_t}\right]\\
 &\le 6\alpha^2\rho + \alpha\E\left[
 \sum_{t=2}^{T-1}(\chi_{t+1}-\chi_t)
 \inner{\phi_{2,t+1}-\phi_{2,t}}{b_t}\right]+
 \alpha\left|
 \sum_{t=2}^{T-1}\sum_{i=1}^K
 (\phi_{2,t+1,i}-\phi_{2,t,i})
 \E[\chi_tb_{t,i}]
 \right| \\
 &\leq 12\alpha^2\rho + \alpha\left|
 \sum_{t=2}^{T-1}\sum_{i=1}^K
 (\phi_{2,t+1,i}-\phi_{2,t,i})
 \E[\chi_tb_{t,i}]
 \right|,
\end{align*}
where the last inequality follows from the fact that $\left|
 \inner{\phi_{2,t+1}-\phi_{2,t}}{b_t}
 \right|
 \le
 \|\phi_{2,t+1}-\phi_{2,t}\|_\infty\|b_t\|_1
 \le\frac32\cdot4\alpha
 =6\alpha$ and
\begin{align*}
 \left|
 \E\left[\sum_{t=2}^{T-1}(\chi_{t+1}-\chi_t)
 \inner{\phi_{2,t+1}-\phi_{2,t}}{b_t}\right]
 \right|&=
 \left|
 \E\left[\sum_{t=2}^{T-1}\chi_t(1-\chi_{t+1})
 \inner{\phi_{2,t+1}-\phi_{2,t}}{b_t}\right]
 \right|\\
 &\le6\alpha\E\left[
 \sum_{t=2}^{T-1}\chi_t(1-\chi_{t+1})\right]
 \le6\alpha\rho,
\end{align*}
where the equality uses $b_s=0$, and the last inequality follows from
$\sum_{t=2}^{T-1}\chi_t(1-\chi_{t+1})
 \le\ind\{s\le T\}.$ Further applying Cauchy--Schwarz inequality gives
\begin{align*}
 &\alpha\left|
 \sum_{t=2}^{T-1}\sum_{i=1}^K
 (\phi_{2,t+1,i}-\phi_{2,t,i})
 \E[\chi_tb_{t,i}]
 \right|\\
 &\le
 \alpha\left(
 \sum_{t=2}^{T-1}\sum_{i=1}^K
 \frac{\bigl(\E[\chi_tb_{t,i}]\bigr)^2}
 {\E[\chi_t(\ind\{I_t=i\}+x_{t,i})]}
 \right)^{1/2}
 \left(
 \sum_{t=2}^{T-1}\sum_{i=1}^K
 \E[\chi_t(\ind\{I_t=i\}+x_{t,i})]
 (\phi_{2,t+1,i}-\phi_{2,t,i})^2
 \right)^{1/2} \\
 &\leq 2\sum_{t=2}^{T-1}\sum_{i=1}^K
 \frac{\bigl(\E[\chi_tb_{t,i}]\bigr)^2}
 {\E[\chi_t(\ind\{I_t=i\}+x_{t,i})]} + \frac{\alpha^2}{8}\sum_{t=2}^{T-1}\sum_{i=1}^K
 \E[\chi_t(\ind\{I_t=i\}+x_{t,i})]
 (\phi_{2,t+1,i}-\phi_{2,t,i})^2.
\end{align*}
Because $\chi_t$ is measurable before $I_t$ is drawn, we have $\E[\chi_tb_{t,i}]
 =
 \E[\chi_t(\ind\{I_t=i\}-x_{t,i})].$ Therefore, using Lemma~\ref{lem:appendix-stopped-recursion}, the first term is bounded as
\begin{align*}
2
 \sum_{t=2}^{T-1}\sum_{i=1}^K
 \frac{\bigl(\E[\chi_tb_{t,i}]\bigr)^2}
 {\E[\chi_t(\ind\{I_t=i\}+x_{t,i})]}
  =2
 \sum_{t=2}^{T-1}\sum_{i=1}^K
 \frac{\bigl(\E[\chi_t(\ind\{I_t=i\}-x_{t,i})]\bigr)^2}
 {\E[\chi_t(\ind\{I_t=i\}+x_{t,i})]}\le1030\eta^2R+18\alpha^2P+680\alpha\rho.
\end{align*}
Lemma~\ref{lem:appendix-squared-differences} further bounds the second term as follows:
\begin{align*}
 \frac{\alpha^2}{8}\sum_{t=2}^{T-1}\sum_{i=1}^K
 \E[\chi_t(\ind\{I_t=i\}+x_{t,i})]
 (\phi_{2,t+1,i}-\phi_{2,t,i})^2\leq \frac{\alpha^2}{8}
 (12360\eta^2R+13P+200\rho).
\end{align*}
Combining these two bounds, we have
\begin{align*}
 \alpha\E\left[\sum_{t=2}^T\chi_t
 \inner{\phi_{2,t}}{b_{t-1}-b_t}\right]&\le12\alpha^2\rho
 +2(515\eta^2R+9\alpha^2P+340\alpha\rho)+\frac{\alpha^2}{8}
 (12360\eta^2R+13P+200\rho)\\
 &\le1700\eta^2R+2500\eta^2P+20\rho,
\end{align*}
where the last inequality uses $\alpha=8\eta$ and $\eta\le10^{-5}$.
\end{proof}

Next, we control the summation of $\inner{\phi_{2,t}}{x_{t-1}-x_t}$ within the phase.

\begin{lemma}
\label{lem:phasewise-omd}
For every deterministic phase label $j\in\{0,\ldots,T-1\}$,
Algorithm~\ref{alg:adaptive} satisfies
\begin{equation}
 \alpha_j\E\left[\sum_{t=2}^T\chi_{j,t}
 \inner{\phi_{2,t}}{x_{t-1}-x_t}\right]
 \le100\eta_j^2R_j+20\eta_j^2P_j+4\rho_j.
 \label{eq:appendix-omd-final}
\end{equation}
\end{lemma}

\begin{proof}
If phase $j$ is not reached, the left-hand side of
\eqref{eq:appendix-omd-final} is zero. On $\{\bar\chi_t=1\}$, let $\bar x_t$ be the OMD update computed at the end of
round $t-1$, before a possible phase reset; when $\bar\chi_t=0$, set
$\bar x_t=x_{t-1}$. Thus $\bar x_t=x_t$ whenever $\chi_t=1$. Define
\begin{equation*}
 S_x\triangleq\E\left[\sum_{t=2}^T\chi_t\inner{\phi_{2,t}}{x_{t-1}-x_t}\right]
\end{equation*}
and
\begin{equation*}
 \bar S_x\triangleq\E\left[\sum_{t=2}^T\bar\chi_t\inner{\phi_{2,t}}{x_{t-1}-\bar x_t}\right].
\end{equation*}
Since $\bar\chi_t=\chi_t+\delta_t$ and $\bar x_t=x_t$ on
$\{\chi_t=1\}$, their exact difference is
\begin{equation*}
 S_x-\bar S_x=-\E\left[\sum_{t=2}^T\delta_t\inner{\phi_{2,t}}{x_{t-1}-\bar x_t}\right].
\end{equation*}
Using $\|\phi_{2,t}\|_\infty\le3/2$,
$\E[\sum_{t=2}^T\delta_t]\le\rho$, we have $$ |S_x-\bar S_x|
\le\E\left[\sum_{t=2}^T\delta_t
 \|\phi_{2,t}\|_\infty\|x_{t-1}-\bar x_t\|_1\right]
 \le3\E\left[\sum_{t=2}^T\delta_t\right]
 \le3\rho.$$

Following the decomposition \eqref{eq:central-S1}--\eqref{eq:central-S2},
write $g_t=\phi_{2,t}-\phi_{2,t-1}$ and define
\begin{equation*}
 \bar S_{x,1}\triangleq\E\left[\sum_{t=2}^T\bar\chi_t\inner{\phi_{2,t-1}}{x_{t-1}-\bar x_t}\right]
\end{equation*}
and
\begin{equation*}
 \bar S_{x,2}\triangleq\E\left[\sum_{t=2}^T\bar\chi_t\inner{g_t}{x_{t-1}-\bar x_t}\right].
\end{equation*}
We have $\bar S_x=\bar S_{x,1}+\bar S_{x,2}$.

We first control $\bar S_{x,1}$ by making the comparison with
Section~\ref{sec:known} explicit. Equation~\eqref{eq:known-coordinate-response}
is unchanged. In the passage from \eqref{eq:known-paired-response-upper} to
\eqref{eq:known-response-upper}, replacing the $1$-Lipschitz function
$\varphi$ by the $2$-Lipschitz function $\varphi_2$ doubles the coefficient
$2.2$ to $4.4$. Equation~\eqref{eq:known-response-lower} is unchanged. The
phase-sampler bound $p_{t-1,i}\ge x_{t-1,i}/(1+2\alpha)$ and
$\eta\le10^{-5}$ also preserve the $0.01$ coordinatewise stability used in
\eqref{eq:known-branch-weight-comparison}. Therefore the final comparison
used to prove \eqref{eq:direct-response} becomes
\begin{equation*}
 \E_{t-1}[\inner{\phi_{2,t-1}}{x_{t-1}-\bar x_t}]\le\frac{4.4}{0.94}\E_{t-1}[D_{t-1}+D_\Psi(\bar x_t,x_{t-1})]\le16\E_{t-1}[D_{t-1}].
\end{equation*}
Here the second inequality uses
$D_\Psi(\bar x_t,x_{t-1})\le2D_{t-1}$, as shown in the last step of the
proof of \eqref{eq:direct-response}. Further, since
$\bar\chi_t=\ind\{\mathsf J_{t-1}=j\}$ is $\mathcal F_{t-2}$-measurable, the
tower property and \eqref{eq:adaptive-D-upper} give
\begin{align*}
 \bar S_{x,1} &=\E\left[\sum_{t=2}^T\bar\chi_t\E_{t-1}\left[\inner{\phi_{2,t-1}}{x_{t-1}-\bar x_t}\right]\right]\tag{$\bar{\chi}_t$ is $\mathcal F_{t-2}$-measurable}\\
 &\le16\E\left[\sum_{t=2}^T\bar\chi_t\E_{t-1}[D_{t-1}]\right]\\
 &\le9.6\eta\E\left[\sum_{t=2}^T\bar\chi_t r_{t-1}^2\right]\tag{using \eqref{eq:adaptive-D-upper}}\\
 &\le9.6\eta\E\left[\ind\{s\le T\}+\sum_{t=2}^T\chi_t r_t^2\right]\\
 &=9.6\eta(R+\rho).
\end{align*}

We next control $\bar S_{x,2}$. Set $q_t\triangleq\inner{x_{t-1}}{g_t^{\odot2}}.$ Applying \eqref{eq:weighted-response-main} on round $t-1$ with $z=g_t$,
and then applying \eqref{eq:adaptive-D-upper}, gives
\begin{equation*}
 \left|\inner{g_t}{x_{t-1}-\bar x_t}\right|\le1.5\sqrt{\eta D_{t-1}}\sqrt{q_t}\le1.5\sqrt{0.6}\eta|r_{t-1}|\sqrt{q_t}\le1.2\eta|r_{t-1}|\sqrt{q_t}.
\end{equation*}
Therefore, we have
\begin{align}
 |\bar S_{x,2}|
 &\le1.2\eta\E\left[\sum_{t=2}^T\bar\chi_t|r_{t-1}|\sqrt{q_t}\right]\notag\\
 &\le1.2\eta\sqrt{\E\left[\sum_{t=2}^T\bar\chi_t r_{t-1}^2\right]\E\left[\sum_{t=2}^T\bar\chi_tq_t\right]}\tag{Cauchy--Schwarz}\\
 &\le1.2\eta\sqrt{(R+\rho)\E\left[\sum_{t=2}^T\bar\chi_tq_t\right]}.
 \label{eq:appendix-omd-cauchy}
\end{align}

The $2$-Lipschitz property of $\varphi_2$ gives $|g_{t,i}|=\left|\varphi_2(\ell_{t,i})-\varphi_2(\ell_{t-1,i})\right|\le2|\ell_{t,i}-\ell_{t-1,i}|$ and $ g_{t,i}^2\le4(\ell_{t,i}-\ell_{t-1,i})^2.$
Together with $x_{t-1,i}\le(1+2\alpha)p_{t-1,i}$, this yields
\begin{align*}
 q_t
 &=\inner{x_{t-1}}{g_t^{\odot2}}
 =\sum_{i=1}^Kx_{t-1,i}g_{t,i}^2\le4\sum_{i=1}^Kx_{t-1,i}(\ell_{t,i}-\ell_{t-1,i})^2\le4(1+2\alpha)\sum_{i=1}^Kp_{t-1,i}(\ell_{t,i}-\ell_{t-1,i})^2.
\end{align*}
Since $\bar\chi_t$ is $\mathcal F_{t-2}$-measurable, we have
\begin{align}
 \E\left[\sum_{t=2}^T\bar\chi_tq_t\right]
 &\le4(1+2\alpha)\E\left[\sum_{t=2}^T\bar\chi_t
 \sum_{i=1}^Kp_{t-1,i}(\ell_{t,i}-\ell_{t-1,i})^2\right]\notag\\
 &=4(1+2\alpha)\E\left[\sum_{t=2}^T\bar\chi_t
 \E_{t-1}\left[(\ell_{t,I_{t-1}}-\ell_{t-1,I_{t-1}})^2\right]\right]\notag\\
 &=4(1+2\alpha)\E\left[\sum_{t=2}^T\bar\chi_t
 (\ell_{t,I_{t-1}}-\ell_{t-1,I_{t-1}})^2\right]
 \tag{$\bar\chi_t$ is $\mathcal F_{t-2}$-measurable}\\
 &=4(1+2\alpha)\left(
 P+\E\left[\sum_{t=2}^T\delta_t
 (\ell_{t,I_{t-1}}-\ell_{t-1,I_{t-1}})^2\right]\right)\notag\\
 &\le4(1+2\alpha)(P+\rho)\notag\\
 &\le5(P+\rho).
 \label{eq:appendix-omd-energy}
\end{align}
Combining the above inequalities together, we obtain that
\begin{align}
 S_x
 \le\bar S_x+3\rho=\bar S_{x,1}+\bar S_{x,2}+3\rho\le9.6\eta(R+\rho)+1.2\eta\sqrt{5(R+\rho)(P+\rho)}+3\rho.
 \label{eq:appendix-omd-contribution}
\end{align}
Since $\alpha=8\eta$, the AM-GM inequality gives
\begin{equation*}
 9.6\sqrt5\eta^2\sqrt{(R+\rho)(P+\rho)}
 \le4.8\sqrt5\eta^2(R+P+2\rho)
 \le10.8\eta^2(R+P+2\rho).
\end{equation*}
Further using the fact that $\eta\leq 10^{-5}$ gives
\begin{align*}
 \alpha S_x\le76.8\eta^2(R+\rho)
 +10.8\eta^2(R+P+2\rho)
 +24\eta\rho\le100\eta^2R+20\eta^2P+4\rho,
\end{align*}
\end{proof}

Now we are ready to prove Lemma~\ref{lem:phase-localization}.

\begin{proof}[Proof of Lemma~\ref{lem:phase-localization}]
By the definitions of $\bar\chi_t$ and $\delta_t$,
$\bar\chi_t=\chi_t+\delta_t$. The indicator $\delta_t$ marks the unique
round immediately after phase $j$ ends. Therefore we have $\E\left[\sum_{t=2}^T\delta_t\right]\le\rho$. 

Next, as $\bar\chi_t$ is measurable before $I_{t-1}$ is drawn, but
$\chi_t$ is not, the two conditional-expectation identities are
$\E[\bar\chi_t\varphi_2(\ell_{t,I_{t-1}})]
 =\E[\bar\chi_t\inner{\phi_{2,t}}{p_{t-1}}]$ and $\E[\chi_t\varphi_2(\ell_{t,I_t})]
 =\E[\chi_t\inner{\phi_{2,t}}{p_t}]$. Further using $\bar\chi_t=\chi_t+\delta_t$ gives
\begin{align}
 &\alpha\E\left[\sum_{t=2}^T\chi_t
 \left(\varphi_2(\ell_{t,I_{t-1}})
 -\varphi_2(\ell_{t,I_t})\right)\right]=\alpha\E\left[\sum_{t=2}^T\chi_t
 \inner{\phi_{2,t}}{p_{t-1}-p_t}\right]+\alpha\E\left[\sum_{t=2}^T\delta_t
 \left(\inner{\phi_{2,t}}{p_{t-1}}
 -\varphi_2(\ell_{t,I_{t-1}})\right)\right].
 \label{eq:appendix-signed-transport}
\end{align}
Thus \eqref{eq:appendix-signed-transport} is the stopped counterpart of
\eqref{eq:expected-signed-transport}: the second term is precisely the
boundary residual created when the predictable indicator $\bar\chi_t$ is
replaced by $\chi_t$. Since $0\le\varphi_2(z)\le3/2$ on $[0,1]$, the
absolute value of its summand is at most $3\delta_t/2$, and
$\E[\sum_{t=2}^T\delta_t]\le\rho$ bounds the entire residual by
$3\alpha\rho/2$. Substituting $p_t=x_t+b_t$ therefore gives
\begin{align}
 &\alpha\E\left[\sum_{t=2}^T\chi_t
 \left(\varphi_2(\ell_{t,I_{t-1}})
 -\varphi_2(\ell_{t,I_t})\right)\right]\le\alpha\E\left[\sum_{t=2}^T\chi_t
 \inner{\phi_{2,t}}{x_{t-1}-x_t}\right]
 +\alpha\E\left[\sum_{t=2}^T\chi_t
 \inner{\phi_{2,t}}{b_{t-1}-b_t}\right]
 +\frac32\alpha\rho.
 \label{eq:appendix-section3-split}
\end{align}
Lemma~\ref{lem:phasewise-omd} bounds the first term on the right-hand side
of \eqref{eq:appendix-section3-split}, and
Lemma~\ref{lem:phasewise-recent-arm} bounds the second. Their sum is at most
$1800\eta^2R+2520\eta^2P+(24+3\alpha/2)\rho$, which is bounded by the
right-hand side of \eqref{eq:appendix-phase-localization}. This proves the
lemma.
\end{proof}

\subsection{Auxiliary Lemmas}
\label{app:phasewise-auxiliary}

The following auxiliary lemmas are used in the proof of
Lemma~\ref{lem:phase-localization}. Throughout this subsection, $\E$ denotes
expectation with respect to all randomness generated by
Algorithm~\ref{alg:adaptive}, with the oblivious loss sequence fixed.

The first lemma records the properties of the ratio functional used in the
proof of Lemma~\ref{lem:phasewise-recent-arm}.

\begin{lemma}
\label{lem:appendix-ratio-functional}
For any jointly distributed random arm $I\in[K]$, random probability vector
$x\in\Delta_K$, and random weight $w\in[0,1]$, define
\begin{align}
 \mathcal N(w;I,x)
 \triangleq\sum_{i=1}^K
 \frac{\bigl(\E[w(\ind\{I=i\}-x_i)]\bigr)^2}
 {\E[w(\ind\{I=i\}+x_i)]}.
 \label{eq:appendix-ratio-functional}
\end{align}
If
$0\le w'\le w\le1$, then we have $\mathcal N(w';I,x)
 \le\mathcal N(w;I,x)+6\E[w-w']$ and $ \mathcal N(w;I,x)\le2\E[w].$
Moreover, whenever $w_1+w_2\le1$, we have $
 \mathcal N(w_1+w_2;I,x)
 \le\mathcal N(w_1;I,x)+\mathcal N(w_2;I,x).$
\end{lemma}

\begin{proof}
For every nonnegative random variable $u$, define $\mu_i(u)\triangleq\E[u(\ind\{I=i\}-x_i)]$ and $
 \nu_i(u)\triangleq\E[u(\ind\{I=i\}+x_i)].$
The pointwise inequality
$|\ind\{I=i\}-x_i|\le\ind\{I=i\}+x_i$ gives
\begin{align*}
 |\mu_i(u)|
 &\le\E[u|\ind\{I=i\}-x_i|]
 \le\nu_i(u).
\end{align*}
Direct calculation shows that $\mu_i(w)=\mu_i(w')+\mu_i(w-w')$ and $\nu_i(w)=\nu_i(w')+\nu_i(w-w').$ If $\nu_i(w')=0$, then $\mu_i(w')=0$ and for any $w\geq w'$
\begin{align*}
 \frac{\mu_i(w')^2}{\nu_i(w')}
 -\frac{\mu_i(w)^2}{\nu_i(w)}
 =-\frac{\mu_i(w)^2}{\nu_i(w)}
 \le0
 \le3\nu_i(w-w').
\end{align*}

If $\nu_i(w')>0$, then we have
\begin{align*}
 \frac{\mu_i(w')^2}{\nu_i(w')}
 -\frac{\mu_i(w)^2}{\nu_i(w)}=&\frac{
 \nu_i(w-w')\mu_i(w')^2
 -2\nu_i(w')\mu_i(w')\mu_i(w-w')
 -\nu_i(w')\mu_i(w-w')^2}
 {\nu_i(w')\nu_i(w)}\\
 &\le
 \frac{\nu_i(w-w')|\mu_i(w')|^2
 +2\nu_i(w')|\mu_i(w')||\mu_i(w-w')|}{\nu_i(w')\nu_i(w)}\\
 &\le
 \frac{3\nu_i(w')^2\nu_i(w-w')}
 {\nu_i(w')\nu_i(w)}
 \le
 \frac{3\nu_i(w')^2\nu_i(w-w')}
 {\nu_i(w')^2}
 =3\nu_i(w-w').
\end{align*}

Summing over $i$ gives
$
 \mathcal N(w';I,x)-\mathcal N(w;I,x)
 \le3\sum_{i=1}^K\nu_i(w-w')
 =6\E[w-w'].
$
The second inequality follows from
\begin{align*}
 \mathcal N(w;I,x)
 =\sum_{i=1}^K\frac{\mu_i(w)^2}{\nu_i(w)}
 \le\sum_{i=1}^K\nu_i(w)
 =2\E[w],
\end{align*}
where we use the fact that $x\in\Delta_K$. Finally, Cauchy-Schwarz inequality leads to the following
\begin{align*}
 \mathcal N(w_1+w_2;I,x)
 =\sum_{i=1}^K
 \frac{(\mu_i(w_1)+\mu_i(w_2))^2}
 {\nu_i(w_1)+\nu_i(w_2)}\le\sum_{i=1}^K
 \left(
 \frac{\mu_i(w_1)^2}{\nu_i(w_1)}
 +\frac{\mu_i(w_2)^2}{\nu_i(w_2)}
 \right)=\mathcal N(w_1;I,x)+\mathcal N(w_2;I,x).
\end{align*}

\end{proof}
Define $\mathcal N_t^{\mathrm{cur}}
\triangleq\mathcal N(\chi_t;I_t,x_t)$ and
$\mathcal N_t^{\mathrm{prev}}
\triangleq\mathcal N(\chi_t;I_{t-1},x_t)$. Since $\chi_t$ is
$\mathcal F_{t-1}$-measurable, we have
\begin{align*}
 \E[\chi_t(\ind\{I_t=i\}-x_{t,i})] =\E\left[\chi_t\E_t[\ind\{I_t=i\}-x_{t,i}]\right]=\E[\chi_t(p_{t,i}-x_{t,i})]
 =\E[\chi_tb_{t,i}].
\end{align*}

The following lemma bounds the current squared mean norm in terms of its
previous-arm counterpart.

\begin{lemma}
\label{lem:appendix-current-norm}
For every fixed round $t\in\{2,\ldots,T\}$,
\begin{align}
 \mathcal N_t^{\mathrm{cur}}
 \le13\alpha\mathcal N_t^{\mathrm{prev}}
 +8\alpha^2\E[\chi_t(r_t^2+v_t^2)].
 \label{eq:appendix-current-bound}
\end{align}
\end{lemma}

\begin{proof}

On the event $\{\chi_t=1\}$, the sampling rule in
Algorithm~\ref{alg:adaptive} gives
\begin{align*}
 \lambda_t
 &=\alpha(2-c_{t-1})
 =\alpha(2-\ell_{t-1,I_{t-1}}),\\
 p_{t,i}
 &=\frac{x_{t,i}+\lambda_t\ind\{I_{t-1}=i\}}{1+\lambda_t}=x_{t,i}
 +\frac{\alpha(2-\ell_{t-1,I_{t-1}})}
 {1+\alpha(2-\ell_{t-1,I_{t-1}})}
 \bigl(\ind\{I_{t-1}=i\}-x_{t,i}\bigr).
\end{align*}
Define
$\gamma_{t,i}\triangleq
\alpha(2-\ell_{t-1,i})/
(1+\alpha(2-\ell_{t-1,i}))$.
Thus,
\begin{align*}
 p_{t,i}-x_{t,i}
 &=\gamma_{t,I_{t-1}}
 \bigl(\ind\{I_{t-1}=i\}-x_{t,i}\bigr),\\
 p_{t,i}+x_{t,i}
 &=(2-\gamma_{t,I_{t-1}})x_{t,i}
 +\gamma_{t,I_{t-1}}\ind\{I_{t-1}=i\}.
\end{align*}
Since $\chi_t$ is measurable before $I_t$ is drawn, we have
\begin{align*}
 \E[\chi_t(\ind\{I_t=i\}-x_{t,i})]
 &=\E[\chi_t(p_{t,i}-x_{t,i})]=\E\left[\chi_t\gamma_{t,I_{t-1}}
 \bigl(\ind\{I_{t-1}=i\}-x_{t,i}\bigr)\right],\\
 \E[\chi_t(\ind\{I_t=i\}+x_{t,i})]
 &=\E[\chi_t(p_{t,i}+x_{t,i})]=\E\left[\chi_t\left(
 (2-\gamma_{t,I_{t-1}})x_{t,i}
 +\gamma_{t,I_{t-1}}\ind\{I_{t-1}=i\}\right)\right].
\end{align*}
Therefore, by the definition of $\mathcal N_t^{\mathrm{cur}}$,
\begin{align*}
 \mathcal N_t^{\mathrm{cur}}
 =\sum_{i=1}^K
 \frac{\left(
 \E\left[\chi_t\gamma_{t,I_{t-1}}
 (\ind\{I_{t-1}=i\}-x_{t,i})\right]\right)^2}
 {\E\left[\chi_t\left(
 (2-\gamma_{t,I_{t-1}})x_{t,i}
 +\gamma_{t,I_{t-1}}\ind\{I_{t-1}=i\}\right)\right]}.
\end{align*}

We then decompose $\mathcal N_t^{\mathrm{cur}}$ using
$(u+v)^2\le2u^2+2v^2$:
\begin{align*}
 \mathcal N_t^{\mathrm{cur}}
 \le
 \underbrace{
 2\sum_{i=1}^K
 \frac{\gamma_{t,i}^2
 \bigl(\E[\chi_t(\ind\{I_{t-1}=i\}-x_{t,i})]\bigr)^2}
 {\E\left[\chi_t\left(
 (2-\gamma_{t,I_{t-1}})x_{t,i}
 +\gamma_{t,I_{t-1}}\ind\{I_{t-1}=i\}
 \right)\right]}
 }_{\text{previous-arm contribution}}+
 \underbrace{
 2\sum_{i=1}^K
 \frac{\bigl(\E[\chi_tx_{t,i}
 (\gamma_{t,i}-\gamma_{t,I_{t-1}})]\bigr)^2}
 {\E\left[\chi_t\left(
 (2-\gamma_{t,I_{t-1}})x_{t,i}
 +\gamma_{t,I_{t-1}}\ind\{I_{t-1}=i\}
 \right)\right]}
 }_{\text{remainder}}.
\end{align*}

We first bound the previous-arm contribution. Since
$
 \E\left[\chi_t\left(
 (2-\gamma_{t,I_{t-1}})x_{t,i}
 +\gamma_{t,I_{t-1}}\ind\{I_{t-1}=i\}
 \right)\right]\ge\frac{\alpha}{1+\alpha}
 \E[\chi_t(\ind\{I_{t-1}=i\}+x_{t,i})],
$
and $\gamma_{t,i}\le2\alpha/(1+2\alpha)$, we have
\begin{align*}
 2\sum_{i=1}^K
 \frac{\gamma_{t,i}^2
 \bigl(\E[\chi_t(\ind\{I_{t-1}=i\}-x_{t,i})]\bigr)^2}
 {\E\left[\chi_t\left(
 (2-\gamma_{t,I_{t-1}})x_{t,i}
 +\gamma_{t,I_{t-1}}\ind\{I_{t-1}=i\}
 \right)\right]}\le
 2\frac{(2\alpha/(1+2\alpha))^2}
 {\alpha/(1+\alpha)}
 \mathcal N_t^{\mathrm{prev}}
 \le12\alpha\mathcal N_t^{\mathrm{prev}}.
\end{align*}

It remains to control the remainder. Since
$
 \E[\chi_t(\ind\{I_t=i\}+x_{t,i})]
 \ge\E[\chi_tx_{t,i}],
$
Cauchy--Schwarz inequality gives
\begin{align*}
 2\sum_{i=1}^K
 \frac{\bigl(\E[\chi_tx_{t,i}
 (\gamma_{t,i}-\gamma_{t,I_{t-1}})]\bigr)^2}
 {\E[\chi_t(\ind\{I_t=i\}+x_{t,i})]}&\le
 2\sum_{i=1}^K
 \frac{\bigl(\E[\chi_tx_{t,i}
 (\gamma_{t,i}-\gamma_{t,I_{t-1}})]\bigr)^2}
 {\E[\chi_tx_{t,i}]}\\
 &\le
 2\E\left[\chi_t\sum_{i=1}^Kx_{t,i}
 (\gamma_{t,i}-\gamma_{t,I_{t-1}})^2\right]\\
 &\le
 2\alpha^2\E\left[\chi_t\sum_{i=1}^Kx_{t,i}
 (\ell_{t-1,i}-\ell_{t-1,I_{t-1}})^2\right].
\end{align*}
Combining the two parts yields
\begin{align}
 \mathcal N_t^{\mathrm{cur}}
 &\le12\alpha\mathcal N_t^{\mathrm{prev}}
 +2\alpha^2\E\left[\chi_t\sum_{i=1}^Kx_{t,i}
 (\ell_{t-1,i}-\ell_{t-1,I_{t-1}})^2\right].
 \label{eq:appendix-current-first-bound}
\end{align}

To bound the remaining loss discrepancy, conditionally on
$\mathcal F_{t-1}$ draw $I'_t\sim x_t$ and realize the sampling rule by
setting $I_t=I_{t-1}$ with probability $\gamma_{t,I_{t-1}}$ and
$I_t=I'_t$ otherwise. Since
$\gamma_{t,I_{t-1}}\le2\alpha$, we have
\begin{align*}
 \E[\chi_tr_t^2]
=\E\left[\chi_t\gamma_{t,I_{t-1}}v_t^2
 +\chi_t(1-\gamma_{t,I_{t-1}})
 (\ell_{t,I'_t}-\ell_{t-1,I_{t-1}})^2\right]\ge(1-2\alpha)\E[\chi_t(\ell_{t,I'_t}-\ell_{t-1,I_{t-1}})^2].
\end{align*}
Applying Proposition~\ref{prop:aux} with
$A=\E[\chi_tx_{t,i}]$ and
$B=\E[\chi_t\ind\{I_{t-1}=i\}]$, multiplying the resulting inequality by
$(\ell_{t,i}-\ell_{t-1,i})^2$, and summing over $i$ give
\begin{align*}
 \E[\chi_t(\ell_{t,I'_t}-\ell_{t-1,I'_t})^2]
 &=\sum_{i=1}^K\E[\chi_tx_{t,i}]
 (\ell_{t,i}-\ell_{t-1,i})^2\\
 &\le\sum_{i=1}^K\left(
 2\E[\chi_t\ind\{I_{t-1}=i\}]
 +6\frac{\bigl(\E[\chi_t(x_{t,i}-\ind\{I_{t-1}=i\})]\bigr)^2}
 {\E[\chi_t(x_{t,i}+\ind\{I_{t-1}=i\})]}
 \right)(\ell_{t,i}-\ell_{t-1,i})^2\\
 &\le2\sum_{i=1}^K\E[\chi_t\ind\{I_{t-1}=i\}]
 (\ell_{t,i}-\ell_{t-1,i})^2
 +6\sum_{i=1}^K
 \frac{\bigl(\E[\chi_t(\ind\{I_{t-1}=i\}-x_{t,i})]\bigr)^2}
 {\E[\chi_t(\ind\{I_{t-1}=i\}+x_{t,i})]}\\
 &=2\E[\chi_tv_t^2]+6\mathcal N_t^{\mathrm{prev}},
\end{align*}
where the second inequality uses
$(\ell_{t,i}-\ell_{t-1,i})^2\le1$.
Consequently, we can obtain that
\begin{align}
 \E\left[\chi_t\sum_{i=1}^Kx_{t,i}
 (\ell_{t-1,i}-\ell_{t-1,I_{t-1}})^2\right]
 &=\E[\chi_t(\ell_{t-1,I'_t}-\ell_{t-1,I_{t-1}})^2]\notag\\
 &\le2\E[\chi_t(\ell_{t,I'_t}-\ell_{t-1,I_{t-1}})^2]
 +2\E[\chi_t(\ell_{t,I'_t}-\ell_{t-1,I'_t})^2]\notag\\
 &\le3\E[\chi_tr_t^2]+4\E[\chi_tv_t^2]
 +12\mathcal N_t^{\mathrm{prev}}.
 \label{eq:appendix-loss-discrepancy}
\end{align}
Here the last inequality uses $2/(1-2\alpha)\le3$.
Substituting \eqref{eq:appendix-loss-discrepancy} into
\eqref{eq:appendix-current-first-bound} gives
\begin{align*}
 \mathcal N_t^{\mathrm{cur}}
 &\le(12\alpha+24\alpha^2)\mathcal N_t^{\mathrm{prev}}
 +6\alpha^2\E[\chi_tr_t^2]+8\alpha^2\E[\chi_tv_t^2]\le13\alpha\mathcal N_t^{\mathrm{prev}}
 +8\alpha^2\E[\chi_t(r_t^2+v_t^2)],
\end{align*}
where the last inequality uses $24\alpha\le1$.
\end{proof}
The next lemma propagates the squared mean norm by one round and sums the
result over the phase.

\begin{lemma}
\label{lem:appendix-stopped-recursion}
Algorithm~\ref{alg:adaptive} satisfies
\begin{align}
 \sum_{t=2}^T\mathcal N_t^{\mathrm{prev}}
 &\le3\sum_{t=2}^T\mathcal N_t^{\mathrm{cur}}
 +3\eta^2R+25\rho,
 \label{eq:appendix-previous-recursion}\\
 \sum_{t=2}^T\mathcal N_t^{\mathrm{cur}}
 &\le515\eta^2R+9\alpha^2P+340\alpha\rho.
 \label{eq:appendix-current-sum}
\end{align}
\end{lemma}

\begin{proof}

By definition, $\chi_t=\ind\{s<t\le\tau\}$ if the phase is entered and otherwise, $\chi_t=0$ for every $t$.
Therefore, pathwise, we have $\sum_{t=1}^{T-1}(1-\chi_t)\chi_{t+1}
 \le\ind\{s\le T\}$ and $
 \sum_{t=1}^{T-1}\chi_t(1-\chi_{t+1})
 \le\ind\{s\le T\}.$

Using
$\chi_{t+1}=\chi_t\chi_{t+1}+(1-\chi_t)\chi_{t+1}$ and
$\chi_t\chi_{t+1}=\chi_t-\chi_t(1-\chi_{t+1})$, Lemma~\ref{lem:appendix-ratio-functional} gives that
\begin{align}
 \mathcal N(\chi_{t+1};I_t,x_t)
 &\le\mathcal N(\chi_t\chi_{t+1};I_t,x_t)
 +\mathcal N((1-\chi_t)\chi_{t+1};I_t,x_t)\notag\\
 &\le\mathcal N_t^{\mathrm{cur}}
 +6\E[\chi_t(1-\chi_{t+1})]
 +2\E[(1-\chi_t)\chi_{t+1}].
 \label{eq:appendix-survivor-ratio}
\end{align}
We next compare
$\mathcal N_{t+1}^{\mathrm{prev}}
=\mathcal N(\chi_{t+1};I_t,x_{t+1})$
with $\mathcal N(\chi_{t+1};I_t,x_t)$. For every $i$,
\begin{align*}
 \E[\chi_{t+1}(\ind\{I_t=i\}-x_{t+1,i})]
 &=\E[\chi_{t+1}(\ind\{I_t=i\}-x_{t,i})]
 -\E[\chi_{t+1}(x_{t+1,i}-x_{t,i})].
\end{align*}

On $\{\chi_{t+1}=1\}$, applying
\eqref{eq:known-coordinatewise-stability} on round $t$ gives
$x_{t+1,i}\ge0.99x_{t,i}$ and hence
\begin{align*}
 \E[\chi_{t+1}(\ind\{I_t=i\}+x_{t+1,i})]
 \ge0.99\E[\chi_{t+1}(\ind\{I_t=i\}+x_{t,i})].
\end{align*}
Therefore, applying AM-GM inequality and the above inequality give the following
\begin{align*}
 \mathcal N_{t+1}^{\mathrm{prev}}
 &\le\frac{2}{0.99}\mathcal N(\chi_{t+1};I_t,x_t)
 +\frac{2}{0.99}\sum_{i=1}^K
 \frac{\bigl(\E[\chi_{t+1}(x_{t+1,i}-x_{t,i})]\bigr)^2}
 {\E[\chi_{t+1}(\ind\{I_t=i\}+x_{t,i})]}.
\end{align*}

To further bound the second term, \eqref{eq:known-relative-motion} applied on round $t$ and Cauchy-Schwarz inequality yield
\begin{align*}
 \sum_{i=1}^K
 \frac{\bigl(\E[\chi_{t+1}(x_{t+1,i}-x_{t,i})]\bigr)^2}
 {\E[\chi_{t+1}(\ind\{I_t=i\}+x_{t,i})]}&\le\sum_{i=1}^K
 \frac{\bigl(\E[\chi_{t+1}|x_{t+1,i}-x_{t,i}|]\bigr)^2}
 {\E[\chi_{t+1}x_{t,i}]}\\
&\le
\sum_{i=1}^K
\mathbb E\!\left[
\chi_{t+1}\frac{(x_{t+1,i}-x_{t,i})^2}{x_{t,i}}
\right] \tag{using Cauchy-Schwarz inequality}\\
&\le\E\left[\chi_{t+1}\sum_{i=1}^K
 \frac{(x_{t+1,i}-x_{t,i})^2}{x_{t,i}^2}\right]\\
 &\le2.1\eta\E[\chi_{t+1}D_t].
\end{align*}
Combining the above two inequalities gives
\begin{align}
 \mathcal N_{t+1}^{\mathrm{prev}}
 \le2.03\mathcal N(\chi_{t+1};I_t,x_t)
 +4.3\eta\E[\chi_{t+1}D_t].
 \label{eq:appendix-previous-one-step}
\end{align}

Moreover, since $
 \sum_{t=1}^{T-1}\chi_{t+1}r_t^2
 \le\ind\{s\le T\}+\sum_{t=2}^T\chi_tr_t^2$, \eqref{eq:adaptive-D-upper} applied on round $t$ gives
\begin{align*}
 \E\left[\sum_{t=1}^{T-1}\chi_{t+1}D_t\right]
 &\le0.6\eta\E\left[\sum_{t=1}^{T-1}\chi_{t+1}r_t^2\right]
 \le0.6\eta(R+\rho).
\end{align*}
Summing \eqref{eq:appendix-survivor-ratio} and
\eqref{eq:appendix-previous-one-step} over $t=1,\ldots,T-1$, using
$\chi_1=0$ and the preceding bounds, yields
\begin{align*}
 \sum_{t=2}^T\mathcal N_t^{\mathrm{prev}}
 &\le2.03\sum_{t=1}^{T-1}\mathcal N(\chi_{t+1};I_t,x_t)
 +4.3\eta\E\left[\sum_{t=1}^{T-1}\chi_{t+1}D_t\right]\\
 &\le2.03\sum_{t=2}^T\mathcal N_t^{\mathrm{cur}}
 +12.18\E\left[\sum_{t=1}^{T-1}\chi_t(1-\chi_{t+1})\right]+4.06\E\left[\sum_{t=1}^{T-1}(1-\chi_t)\chi_{t+1}\right]
 +4.3\eta\E\left[\sum_{t=1}^{T-1}\chi_{t+1}D_t\right]\\
 &\le2.03\sum_{t=2}^T\mathcal N_t^{\mathrm{cur}}
 +16.24\rho+2.58\eta^2(R+\rho)\\
 &\le3\sum_{t=2}^T\mathcal N_t^{\mathrm{cur}}
 +3\eta^2R+25\rho,
\end{align*}
which proves \eqref{eq:appendix-previous-recursion}.

Finally, summing \eqref{eq:appendix-current-bound} over $t$ and applying
\eqref{eq:appendix-previous-recursion} give
\begin{align*}
 \sum_{t=2}^T\mathcal N_t^{\mathrm{cur}}
 &\le13\alpha\sum_{t=2}^T\mathcal N_t^{\mathrm{prev}}
 +8\alpha^2(R+P)\le39\alpha\sum_{t=2}^T\mathcal N_t^{\mathrm{cur}}
 +(39\alpha\eta^2+8\alpha^2)R
 +8\alpha^2P+325\alpha\rho.
\end{align*}
Rearranging and using $\alpha=8\eta$ and $\eta\le10^{-5}$ yield
\begin{align*}
 \sum_{t=2}^T\mathcal N_t^{\mathrm{cur}}
 &\le\frac{39\alpha\eta^2+8\alpha^2}{1-39\alpha}R
 +\frac{8\alpha^2}{1-39\alpha}P
 +\frac{325\alpha}{1-39\alpha}\rho\le515\eta^2R+9\alpha^2P+340\alpha\rho,
\end{align*}
which proves \eqref{eq:appendix-current-sum}.
\end{proof}

The following lemma bounds the weighted sum of squared differences used after
summation by parts.

\begin{lemma}
\label{lem:appendix-squared-differences}
Let $g_{t+1,i}\triangleq\phi_{2,t+1,i}-\phi_{2,t,i}$ and
$q_{t,i}\triangleq\E[\chi_t(\ind\{I_t=i\}+x_{t,i})]$. Then
\begin{align}
 \sum_{t=2}^{T-1}\sum_{i=1}^Kq_{t,i}g_{t+1,i}^2
 \le12360\eta^2R+13P+200\rho.
 \label{eq:appendix-squared-difference-sum}
\end{align}
\end{lemma}

\begin{proof}
The inequalities
$\chi_t\le\chi_{t+1}+\chi_t(1-\chi_{t+1})$,
$|g_{t+1,i}|\le3/2$, and
$\sum_{i=1}^K(\ind\{I_t=i\}+x_{t,i})=2$ give
\begin{align*}
 \E\left[\sum_{t=2}^{T-1}\chi_t(1-\chi_{t+1})
 \sum_{i=1}^K(\ind\{I_t=i\}+x_{t,i})g_{t+1,i}^2\right]\le\frac{9}{2}\E\left[\sum_{t=2}^{T-1}\chi_t(1-\chi_{t+1})\right]\le7\rho.
\end{align*}
By the definition of $q_{t,i}$ and
$\chi_t\le\chi_{t+1}+\chi_t(1-\chi_{t+1})$,
\begin{align*}
 \sum_{t=2}^{T-1}\sum_{i=1}^Kq_{t,i}g_{t+1,i}^2
 &=\sum_{t=2}^{T-1}\sum_{i=1}^K
 \E[\chi_t(\ind\{I_t=i\}+x_{t,i})]g_{t+1,i}^2\\
 &\le\sum_{t=2}^{T-1}\sum_{i=1}^K
 \E[\chi_{t+1}(\ind\{I_t=i\}+x_{t,i})]g_{t+1,i}^2+\E\left[\sum_{t=2}^{T-1}\chi_t(1-\chi_{t+1})
 \sum_{i=1}^K(\ind\{I_t=i\}+x_{t,i})g_{t+1,i}^2\right]\\
 &\le7\rho
 +\sum_{t=2}^{T-1}\sum_{i=1}^K
 \E[\chi_{t+1}\ind\{I_t=i\}]g_{t+1,i}^2+\sum_{t=2}^{T-1}\sum_{i=1}^K
 \E[\chi_{t+1}x_{t,i}]g_{t+1,i}^2,
\end{align*}
Since $\varphi_2$ is $2$-Lipschitz,
$g_{t+1,i}^2\le4(\ell_{t+1,i}-\ell_{t,i})^2$. Hence, we have
\begin{align*}
 \sum_{t=2}^{T-1}\sum_{i=1}^K
 \E[\chi_{t+1}\ind\{I_t=i\}]g_{t+1,i}^2
 &\le4\E\left[\sum_{t=2}^{T-1}
 \chi_{t+1}(\ell_{t+1,I_t}-\ell_{t,I_t})^2\right]=4\E\left[\sum_{u=3}^T
 \chi_u(\ell_{u,I_{u-1}}-\ell_{u-1,I_{u-1}})^2\right]
 \le4P.
\end{align*}

Applying Proposition~\ref{prop:aux} with
$A=\E[\chi_{t+1}x_{t,i}]$ and
$B=\E[\chi_{t+1}\ind\{I_t=i\}]$, and using
$g_{t+1,i}^2\le4(\ell_{t+1,i}-\ell_{t,i})^2\le4$, give
\begin{align*}
 \E[\chi_{t+1}x_{t,i}]g_{t+1,i}^2
 &\le8\E[\chi_{t+1}\ind\{I_t=i\}]
 (\ell_{t+1,i}-\ell_{t,i})^2+24
 \frac{\bigl(\E[\chi_{t+1}(\ind\{I_t=i\}-x_{t,i})]\bigr)^2}
 {\E[\chi_{t+1}(\ind\{I_t=i\}+x_{t,i})]}.
\end{align*}
Summing over $t\in\{2,\dots,T-1\}$ and $i\in[K]$ yields
\begin{align*}
 \sum_{t=2}^{T-1}\sum_{i=1}^K
 \E[\chi_{t+1}x_{t,i}]g_{t+1,i}^2
 &\le8\E\left[\sum_{t=3}^T
 \chi_t(\ell_{t,I_{t-1}}-\ell_{t-1,I_{t-1}})^2\right]+24\sum_{t=2}^{T-1}
 \mathcal N(\chi_{t+1};I_t,x_t)\\
 &\le8P+24\sum_{t=2}^{T-1}
 \mathcal N(\chi_{t+1};I_t,x_t).
\end{align*}
Summing \eqref{eq:appendix-survivor-ratio} over
$t=2,\ldots,T-1$ gives
\begin{align*}
 \sum_{t=2}^{T-1}\mathcal N(\chi_{t+1};I_t,x_t)
 &\le\sum_{t=2}^{T-1}\mathcal N_t^{\mathrm{cur}}
 +6\E\left[\sum_{t=2}^{T-1}\chi_t(1-\chi_{t+1})\right]
 +2\E\left[\sum_{t=2}^{T-1}(1-\chi_t)\chi_{t+1}\right]\le\sum_{t=2}^{T-1}\mathcal N_t^{\mathrm{cur}}+8\rho.
\end{align*}
Combining the preceding bounds and applying
\eqref{eq:appendix-current-sum} yield
\begin{align*}
 \sum_{t=2}^{T-1}\sum_{i=1}^Kq_{t,i}g_{t+1,i}^2
 &\le7\rho+4P+8P
 +24\sum_{t=2}^{T-1}\mathcal N(\chi_{t+1};I_t,x_t)\\
 &\le12P+24\sum_{t=2}^{T-1}\mathcal N_t^{\mathrm{cur}}+199\rho\\
 &\le12360\eta^2R+(12+216\alpha^2)P+(199+8160\alpha)\rho\\
 &\le12360\eta^2R+13P+200\rho,
\end{align*}
where $\alpha=8\eta\le8\cdot10^{-5}$ implies
$216\alpha^2\le1$ and $8160\alpha\le1$.
\end{proof}
\begin{proposition}\label{prop:aux}
For any $A,B\geq 0$, we have $
 A\le2B+6\frac{(A-B)^2}{A+B}$.
\end{proposition}
\begin{proof}
    
The inequality is immediate when $A\le2B$. When $A>2B$, we have
$
 6\frac{(A-B)^2}{A+B}
 \ge6\frac{(A/2)^2}{3A/2}=A.$
\end{proof}

\end{document}